\documentclass{article}

\PassOptionsToPackage{numbers, sort&compress}{natbib}

\usepackage[final]{neurips_2020}

\usepackage[utf8]{inputenc} 
\usepackage[T1]{fontenc}    
\usepackage{hyperref}       
\usepackage{url}            
\usepackage{booktabs}       
\usepackage{amsfonts}       
\usepackage{nicefrac}       
\usepackage{microtype}      
\usepackage{subfig}         

\usepackage{amsmath,amsfonts,bm}

\def\1{\bm{1}}

\def\vone{{\bm{1}}}

\def\ve{{\bm{e}}}

\def\vp{{\bm{p}}}

\def\vu{{\bm{u}}}
\def\vv{{\bm{v}}}

\def\evp{{p}}

\def\evv{{v}}

\def\mA{{\bm{A}}}

\def\mE{{\bm{E}}}

\def\mG{{\bm{G}}}
\def\mH{{\bm{H}}}

\def\mM{{\bm{M}}}

\def\mW{{\bm{W}}}
\def\mX{{\bm{X}}}
\def\mY{{\bm{Y}}}
\def\mZ{{\bm{Z}}}

\DeclareMathAlphabet{\mathsfit}{\encodingdefault}{\sfdefault}{m}{sl}
\SetMathAlphabet{\mathsfit}{bold}{\encodingdefault}{\sfdefault}{bx}{n}

\def\sD{{\mathbb{D}}}

\def\sG{{\mathbb{G}}}
\def\sH{{\mathbb{H}}}

\def\emA{{A}}

\def\emH{{H}}

\def\emM{{M}}

\def\emZ{{Z}}

\newcommand{\R}{\mathbb{R}}

\DeclareMathOperator*{\argmax}{arg\,max}
\DeclareMathOperator*{\argmin}{arg\,min}

\usepackage{cy}
\usepackage{bbm}

\usepackage[normalem]{ulem}

\newcommand{\sfmx}{\sigma_{\rm S}}
\newcommand{\hdmx}{\sigma_{\rm H}}
\newcommand{\relu}{{\rm ReLU}}

\newcommand{\funcdist}{\mathsf{d}}
\newcommand{\attn}{{\rm Attn}}
\newcommand{\tb}{{\rm TB}}
\renewcommand{\indic}[1]{\mathbbm{1}\left\{#1\right\}}

\newcommand{\BB}{\text{BERT}_{\text{BASE}}}

\newcommand{\head}{{\rm Head}}
\newcommand{\sparsehead}{{\rm SHead}}

\newcommand{\sparseattn}{{\rm SAttn}}
\newcommand{\sparsetb}{{\rm STB}}
\renewcommand{\mod}{{\rm mod}}

\usepackage{thmtools}
\usepackage{thm-restate}

\title{$O(n)$ Connections are Expressive Enough: \\Universal Approximability of Sparse Transformers}

\author{%
  Chulhee Yun \\
  MIT\\
  \texttt{chulheey@mit.edu} \\
  \And
  Yin-Wen Chang \\
  Google Research NY \\
  \texttt{yinwen@google.com} \\
  \And
  Srinadh Bhojanapalli \\
  Google Research NY \\
  \texttt{bsrinadh@google.com} \\
  \And
  Ankit Singh Rawat \\
  Google Research NY \\
  \texttt{ankitsrawat@google.com} \\
  \And
  Sashank J.\ Reddi \\
  Google Research NY \\
  \texttt{sashank@google.com} \\
  \And
  Sanjiv Kumar \\
  Google Research NY \\
  \texttt{sanjivk@google.com} \\
}

\begin{document}

\maketitle

\begin{abstract}
{Recently, Transformer networks have redefined the state of the art in many NLP tasks. However, these models suffer from quadratic computational cost in the input sequence length $n$ to compute pairwise attention in each layer. This has prompted recent research into \emph{sparse Transformers} that sparsify the connections in the attention layers.} While empirically promising for long sequences, fundamental questions remain unanswered: Can sparse Transformers approximate any arbitrary sequence-to-sequence function, similar to their dense counterparts? How does the sparsity pattern and the sparsity level affect their performance? In this paper, we address these questions and provide a \emph{unifying framework} that captures existing sparse attention models. We propose sufficient conditions under which we prove that a sparse attention model can \emph{universally approximate} any sequence-to-sequence function. 
{Surprisingly, our results show that sparse Transformers with only $O(n)$ connections per attention layer} can approximate the same function class as the dense model with $n^2$ connections.
Lastly, we present experiments comparing different patterns/levels of sparsity on standard NLP tasks.
\end{abstract}



\vspace*{-5pt}
\section{Introduction}
\vspace*{-4pt}
\label{sec:intro}

Transformer networks \citep{vaswani2017attention} and their variants \citep{xlnet2019} have played a key role in the recent advancement of the state of the art in many natural language processing tasks, such as machine translation \citep{vaswani2017attention}, language modeling \citep{radford2018gpt,radford2019gpt2}, and question answering \citep{devlin2018bert, xlnet2019, roberta2019}. 
{The key component of these networks is the self-attention layer \citep{bahdanau2015neural, luong2015multiplicativ}, which updates the embeddings of the input tokens based on their context.}
Naturally, the self-attention layer also plays the key role in the analysis of Transformers \cite{perez2019turing, brunner2019identifiability, hahn2020theoretical, yun2019transformers, bhojanapalli2020low}; for example, \citet{yun2019transformers} show that Transformers can approximate any continuous sequence-to-sequence functions (i.e., universal approximation), by proving that self-attention layers can compute \emph{contextual mappings} of the input embeddings.

On the other hand, the self-attention layer is also the main bottleneck in scaling these models. It involves computation of \emph{pairwise} inner products between input tokens, which results in quadratic computational complexity $O(n^2)$ in the length of the input sequence $n$. 
To mitigate this issue, researchers have developed methods to \emph{sparsify} the pairwise interactions/connections in self-attention layers to reduce the computational complexity and/or improve model interpretability, and have shown successful empirical results on tasks with long sequence lengths \citep{guo2019star,child2019generating,sukhbaatar2019adaptive,cui2019fine,correia2019adaptively,qiu2019blockwise,ye2019bp,zhao2019explicit,roy2020efficient,li2020sac,beltagy2020longformer,zaheer2020big}. For example, \citet{child2019generating} propose sparse Transformers for sequence generation. One of the sparsity patterns considered in \cite{child2019generating} is the \textsc{Strided} pattern, where the sparse attention layers alternate between two patterns: each token attends to only i) $w$ local neighbors, and then ii) one after every $w$ tokens in a strided manner. By choosing $w = O(\sqrt{n})$, they propose sparse attention layers with $O(n^{3/2})$ connections and show improvements on both speed and performance over the dense Transformer.

In the existing results, the rule of thumb for designing sparsity patterns (e.g., \textsc{Strided}) is connectivity; the intuition is that if each token can attend to the other tokens in multiple ``hops,'' then the resulting sparse Transformers do not lose much expressive power. 
However, there has been \emph{no formal justification} for this intuition. How does sparsifying the interaction in the self-attention layers affect the model's expressive power and ability to learn? What are the sparsity levels at which the model still retains its rich expressive power, and how is it affected by the sparsity pattern? Such fundamental questions about sparse attention models still remain unanswered.

\vspace*{-3pt}
\subsection{Summary of contributions}
\vspace*{-2pt}
In this paper, we take the first step towards a theoretical understanding of sparse Transformers. 
\begin{list}{{\tiny$\bullet$}}{\leftmargin=1.8em}
  \setlength{\itemsep}{1pt}
\item We propose a unified framework to analyze sparse Transformers, which generalizes the existing approaches that sparsify attention layers (\S~\ref{sec:framework}).
\item We propose a set of intuitive conditions on the sparsity pattern (Assumption~\ref{assm:spsptrn}) and the probability map (Assumption~\ref{assm:rhotohdmx}). Then, in Theorem~\ref{thm:main}, we show that Sparse Transformers, of fixed width and arbitrary depth, satisfying these conditions are universal approximators of any continuous sequence-to-sequence functions {for any given fixed sequence length (\S~\ref{sec:assm} and \S~\ref{sec:thm}).}
\item We next show some examples of existing sparse Transformers \citep{child2019generating, guo2019star, beltagy2020longformer, cui2019fine,correia2019adaptively,zhao2019explicit,zaheer2020big} that satisfy these conditions, and hence have universal approximability (\S~\ref{sec:analyze-existing}). Surprisingly, we show that there are sparse Transformers with only $O(n)$ connections per self-attention layer (instead of $n^2$) that have enough expressive power to approximate arbitrary continuous functions (Corollary~\ref{cor:O-of-n}).
\item We report experimental results on standard NLP tasks using sparse Transformers, comparing different sparsity patterns/levels (\S~\ref{sec:exp}).
\end{list}

\vspace*{-5pt}
\section{Preliminaries and related works}
\vspace*{-4pt}
In this section, we summarize the notation we will use throughout the paper, give a brief overview of Transformers, and then discuss existing efforts to sparsify the self-attention mechanism.

\vspace*{-3pt}
\subsection{Notation}
\vspace*{-2pt}
For a positive integer $a$, we denote $[a] = \{1, 2, \dots, a\}$. 
For any vector $\vv \in \reals^d$, let $\evv_j$ denote its $j$-th coordinate.
For any matrix $\mA \in \reals^{d \times n}$, let $\mA_j$ denote its $j$-th column, and $\mA_{\mc S}$ denote the submatrix consisting of columns of $\mA$ in an index set $\mc S \subseteq [n]$. 
We use $\norm{\mA}_p$ to denote the entry-wise $\ell^p$ norm of $\mA$. 
Let $\sfmx[\cdot]$ be the softmax operator, which takes a matrix as input and applies softmax operation to each column of the matrix, which results in a column stochastic matrix.

\vspace*{-3pt}
\subsection{Transformers and their universal approximation power}
\vspace*{-2pt}
\label{sec:stdtransformer}
A Transformer network, consisting of multiple layers of Transformer blocks, implements a sequence-to-sequence function that maps $\reals^{d \times n}$ to $\reals^{d \times n}$.
A Transformer Block ($\tb$) consists of two layers: a self-attention layer and a token-wise feed-forward layer, and both layers have an identity skip connection. More concretely, for an input $\mX \in \reals^{d \times n}$ consisting of $d$-dimensional embeddings of $n$ tokens, a Transformer block consists of the following two layers:
\begin{subequations}
\label{eq:tb}
\begin{align}
    \attn(\mX) &= \mX + \mW_O 
    \begin{bmatrix}
    \head^1(\mX)\\
    \vdots\\
    \head^h(\mX)
    \end{bmatrix}\!; 
    ~~\head^i(\mX) = \mW_V^i \mX \cdot \sfmx [(\mW_K^i \mX)^T \mW_Q^i \mX]\label{eq:attn}
    \\
    \tb(\mX) &= \attn(\mX) + \mW_2 \cdot \relu (\mW_1 \cdot \attn(\mX)),\label{eq:ff}
\end{align}
\end{subequations}
where $\mW_O \in \reals^{d \times mh}$, $\mW_V^i, \mW_K^i, \mW_Q^i \in \reals^{m \times d}$, $\mW_2 \in \reals^{d \times r},$ and $\mW_1 \in \reals^{r \times d}$. Although our analysis and experiments rely on bias vectors, we omit those in \eqref{eq:tb} for simplicity.

To endow the network with information about the position of input tokens, it is common to add a positional embedding $\mE \in \R^{d \times n}$ to the input $\mX$ before feeding it to the network. The positional embedding $\mE$ can be fixed \citep{vaswani2017attention} or trainable \citep{devlin2018bert}; we consider the latter. 
Using a trainable~$\mE$, $\mc {T}^{h,m,r}$ is defined to be a class of functions of the form $\mX \mapsto t(\mX + \mE)$, where $t$ is a composition of any number of Transformer blocks with $h$ attention heads of head size $m$, and hidden layers of width $r$. Thus, $\mc {T}^{h,m,r}$ is a class of Transformers with a fixed \emph{width} while the depth can be arbitrary.

Further, let $\mc F$ be the class of continuous functions $f: \sD \to \reals^{d \times n}$ defined on any compact domain $\sD \subset \reals^{d \times n}$, where continuity is defined with respect to the entry-wise $\ell_p$ norm ($1\leq p < \infty$). 
\citet[Theorem~3]{yun2019transformers} show that $\mc T^{2,1,4}$ can \emph{universally approximate} $\mc F$. More precisely, for any $f \in \mc F$, $\epsilon >0$ and $1 \leq p < \infty$, there exists a function $g \in \mc T^{2,1,4}$ such that
    $\funcdist_p(f,g) \defeq (\int_{\sD} \norm{f(\mX) - g(\mX)}_p^p d\mX  )^{1/p} \leq \epsilon$.
{
Our goal in this paper is to study, in a similar manner, the expressive power of \emph{sparse} Transformers.}

\vspace*{-3pt}
\subsection{Sparse Transformers}
\vspace*{-2pt}
\label{sec:prevresults}
As seen in Eq.~\eqref{eq:attn}, the self-attention layer involves computing the inner product between each pair of tokens, which we will refer to as the \textbf{attention score matrix} $\mA^i \defeq (\mW_K^i \mX)^T \mW_Q^i \mX \in \R^{n \times n}$. This leads to quadratic computational complexity in $n$, which makes it expensive to apply Transformers to tasks with long sequence lengths.
{One popular approach to mitigate this problem is to \emph{sparsify} the self-attention layers. We sub-classify sparse Transformers into three categories and summarize them below. For a more extensive summary, please see a recent survey~\citep{tay2020efficient}.}

The first category reduces computation by making $\mA^i$ sparse in a \emph{pre-determined} manner. Each token in the sequence only attends to a fixed smaller set of other tokens instead of the whole sequence \citep{child2019generating, qiu2019blockwise, beltagy2020longformer}. In some papers, auxiliary tokens are added to improve connectivity between existing tokens while maintaining sparsity \citep{guo2019star,ye2019bp}. 
One drawback of these approaches is that the sparsity pattern is independent of input, so it cannot adapt to the data. To remedy this issue, \citep{sukhbaatar2019adaptive} proposes to learn local attention span from data.
{In a concurrent paper, \citet{zaheer2020big} propose the \textsc{BigBird} sparsity pattern which falls into this category. For \textsc{BigBird}, the authors show its theoretical properties such as universal approximation and Turing completeness, as well as its superior empirical performance. We note that our paper focuses on universal approximation for a \emph{broader} class of sparse Transformers, by proposing a unifying framework to analyze them.} 

The second category studies making $\mA^i$ sparse \emph{after} the full $\mA^i$ has been computed \citep{cui2019fine,correia2019adaptively,zhao2019explicit}.
Here, the focus is not on the computational gain via sparsity, because the full score matrix $\mA^i$ has to be computed first; rather, the goal here is to make attention layers more interpretable, as well as to improve performance.
This line of works modifies $\sfmx$ in \eqref{eq:attn} to other probability maps, by using top-$k$ elements or adopting sparser variants such as sparselin-gen or $\alpha$-entmax \citep{laha2018controllable,peters2019sparse}. Compared to the first category, this approach has an advantage that sparsity patterns are adaptive to data.

The last category attempts to get the best of both worlds. This line of works tries to learn sparsity patterns from data using extra components predicting the connection between tokens, e.g., $k$-means clustering \citep{roy2020efficient},  LSTM \citep{li2020sac}, or locality-sensitive hashing \citep{kitaev2020reformer}. This way, one can adaptively determine the sparsity patterns before computing the score matrix. However, the drawback of this approach is that one needs extra computation to train/run these additional components, which may be expensive.


\vspace*{-5pt}
\section{Universal approximation theorem for sparse Transformers}
\vspace*{-4pt}
\label{sec:sparseTransformer}
In this section, we derive a unifying framework to study sparse Transformers. We then propose a set of conditions on the sparse self-attention layers, and prove that the sparse Transformers satisfying theses conditions are universal approximators of any continuous sequence-to-sequence functions. Finally, we show some examples of existing sparse Transformers that satisfy these conditions. 

\vspace*{-3pt}
\subsection{A unifying framework for sparse Transformers}
\vspace*{-2pt}
\label{sec:framework}
We modify the Transformer block in \eqref{eq:tb} 
to the following sparse Transformer block ($\sparsetb$):
\begin{align}
    \sparseattn^l(\mX) &= \mX + \mW_O 
    \begin{bmatrix}
    \sparsehead^{1,l}(\mX)\\
    \vdots\\
    \sparsehead^{h,l}(\mX)
    \end{bmatrix}\!,
    ~~\sparsehead^{i,l}(\mX)_k = \mW_V^i \mX_{\mc A_k^l} \cdot \rho [(\mW_K^i \mX_{\mc A_k^l})^T \mW_Q^i \mX_k] \nonumber
    \\
    \sparsetb^l(\mX) &= \sparseattn^l(\mX) + \mW_2 \cdot \relu (\mW_1 \cdot \sparseattn^l(\mX)),\label{eq:s-ff}
\end{align}
where the sets $\mc A_k^l \subseteq [n]$, for $k \in [n]$ and $l \in [p]$, define the $p$ sparsity patterns (formally defined below), which are indexed by $l \in [p]$. 
Moreover, the parameter dimensions stay the same as in \eqref{eq:tb}. 

Note that there are three main modifications from the dense Transformer.
\begin{list}{{\tiny$\bullet$}}{\leftmargin=1.8em}
\item (\emph{Cycling blocks}) There are superscripts $l \in [p]$ added to the symbols such as $\sparseattn$. 
Unlike dense Transformers, some sparse Transformers \textbf{cycle} through $p$ different patterns.
For example, the \textsc{Strided} pattern \citep{child2019generating} described in \S~\ref{sec:intro} alternates between two different patterns, which corresponds to $p = 2$. 
We add the superscript $l$ to include such cases in our formulation. We assume that the layers in a sparse Transformer cycle through $\sparsetb^1, \dots, \sparsetb^p$.

\item (\emph{Sparsity patterns}) Note that $\sparsehead^{i,l}(\mX)_k$ denotes the $k$-th column of the $i$-th sparse attention head. Unlike dense Transformers, the inner product of the $k$-th query vector $\mW_Q^i \mX_k$ is taken only with $\mW_K^i \mX_{\mc A_k^l}$, the key vectors of tokens in the set $\mc A^l_k \subseteq [n]$. Hence, instead of all $n$ tokens, the $k$-th token computes attention scores with only tokens in $\mc A^l_k$. For $l \in [p]$, we refer to the collection of the index sets $\{\mc A^l_k\}_{k\in[n]}$, or simply $\{\mc A^l_k\}$, as a \textbf{sparsity pattern}.
As a result, $\sparsehead^{i,l}(\mX)_k$ is a linear combination of columns in $\mW_V^i \mX_{\mc A_k^l}$, rather than the whole sequence.

\item (\emph{Probability map})
After computing the attention score matrix, the dense Transformer \eqref{eq:tb} uses the softmax operator $\sfmx$ to get a column stochastic matrix. In the sparse Transformers, we generalize $\sfmx$ to $\rho$. The \textbf{probability map} $\rho$ is any map that takes a matrix as input and outputs a column stochastic matrix. 
\end{list}

As a sanity check, by choosing $p=1$, $\mc A_k^1 = [n]$ for all $k \in [n]$, and $\rho = \sfmx$, we recover the dense Transformer \eqref{eq:tb}.
Note also that the sparse Transformer formulation covers the first and second categories of existing results discussed in \S~\ref{sec:prevresults}. The first category corresponds to choosing a predetermined sparsity pattern(s) $\{\mc A_k^l\}$, while setting $\rho = \sfmx$. The second category corresponds to opting for a probability map $\rho$ other than softmax $\sfmx$, while maintaining $\mc A_k^1 = [n]$ for all $k \in [n]$.

In this paper, we assume for simplicity that all sparse attention heads $\sparsehead^{1,l}, \dots, \sparsehead^{h,l}$ in a single layer have identical sparsity patterns $\{\mc A^l_k\}$. However, since our result only requires two sparse attention heads per layer (as we will see in Theorem~\ref{thm:main}), our result can be easily extended to the case that allows multiple sparsity patterns in a single layer.

Similar to $\mc T^{h,m,r}$ in \S~\ref{sec:stdtransformer}, we define the class of functions represented by sparse Transformers. We hide the dependence of this class on the sparsity patterns and probability map to simplify the notation.
\begin{align}
\label{eq:transformer-fclass}
    \mc {ST}^{h,m,r} 
    \defeq 
    \{ \mX \mapsto t(\mX + \mE) \mid
    &~\text{$t$ is a composition of cycling sparse Transformer blocks $\sparsetb^l$,} \nonumber \\
    &~\text{each with $h$ heads of head size $m$ and hidden layer size $r$,} \nonumber \\
    &~\text{and positional embedding } \mE \in \reals^{d \times n} \text{ is trainable} \}.
\end{align}

\vspace*{-3pt}
\subsection{Conditions on sparsity patterns and probability map}
\vspace*{-2pt}
\label{sec:assm}
In this section, we define a set of conditions on the sparsity patterns $\{\mc A^l_k\}$ and the probability map $\rho$ that ensures that the sparse Transformer universally approximate the function class $\mathcal{F}$ (cf.~\S~\ref{sec:stdtransformer}).

For $k \in [n]$ and the index sets $\{\mc A^l_k\}_{l \in [p]}$, 
we define a sequence of sets $\{\mc S^t_k\}_{t \geq 1}$ in a recursive way:
\begin{equation*}
    \mc S^1_k \defeq \mc A^1_k,~~
    \mc S^t_k \defeq \bigcup\nolimits_{j \in \mc A^{(t-1) \text{ mod } p + 1}_k} \mc S^{t-1}_j.
\end{equation*}
The set $\mc S^t_k$ is the set of all tokens that the $k$-th token can directly/indirectly attend to, after $t$ sparse attention layers with sparsity patterns cycling through $\{\mc A^1_k\}, \{\mc A^2_k\}, \dots, \{\mc A^p_k\}$. We now state our conditions on sparsity patterns.
\begin{assumption}
\label{assm:spsptrn}
The sparsity patterns $\{\mc A_k^l\}$ satisfy the following:
\begin{enumerate}
  \setlength{\itemsep}{0pt}
    \item \label{assm:spsptrn-cond1} For all $k \in [n]$ and $l \in [p]$, we have $k \in \mc A_k^l$.
    \item \label{assm:spsptrn-cond2} There exists a permutation $\gamma: [n] \to [n]$ such that, for all $i \in [n-1]$, $\gamma(i) \in \bigcup_{l=1}^p \mc A^l_{\gamma(i+1)}$.
    \item \label{assm:spsptrn-cond3} There exists a finite $s \in \naturals$ such that $s = \min\{u \mid \mc S^{u}_k = [n]~\text{for all}~k \in [n]\}$. 
\end{enumerate}
\end{assumption}
Assumption~\ref{assm:spsptrn}.\ref{assm:spsptrn-cond1} is equivalent to saying that every token always attends to itself.
Assumption~\ref{assm:spsptrn}.\ref{assm:spsptrn-cond2} requires that there is a chain of \emph{direct} connections that covers all $n$ tokens; note that the set $\bigcup_{l=1}^p \mc A^l_{\gamma(i+1)}$ is the set of all tokens that the $\gamma(i+1)$-th token directly attends to. To elaborate more about the chain, consider a directed graph with $n$ vertices corresponding to the $n$ tokens. For any $j \in \bigcup_{l=1}^p \mc A^l_{k}$, we add a directed edge $j \to k$. Given a graph constructed this way, Assumption~\ref{assm:spsptrn}.\ref{assm:spsptrn-cond2} requires that the graph has a Hamiltonian path $\gamma(1) \to \gamma(2) \to \cdots \to \gamma(n)$.
Assumption~\ref{assm:spsptrn}.\ref{assm:spsptrn-cond3} requires that after $s$ sparse attention layers, every token can attend to all the other tokens, either directly or indirectly.

{As we discuss in \S~\ref{sec:analyze-existing}, the statements in Assumption~\ref{assm:spsptrn} are natural enough to be satisfied by many existing sparsity patterns studied in the literature.
In fact, Assumption~\ref{assm:spsptrn}.\ref{assm:spsptrn-cond3} is \emph{necessary} for universal approximation. If $p=1$, $n=2$, $\mc A^1_1 = \{1\}$ and $\mc A^1_2 = \{1,2\}$, then the first token never attends to the second, so this sparse Transformer cannot approximate a function whose first output token is dependent on both input tokens. The other two assumptions are required in parts of our proof, which involve ``propagating information'' over all the tokens in a sequential manner.}

We now state the assumption on the probability map $\rho[\cdot]$. For this, we define $\hdmx[\cdot]$ to be the hardmax operator, which outputs the one-hot representation of the $\argmax$ entry for each column of the input matrix. 
{Since $\rho$ is a column-wise operator that outputs a column-stochastic matrix, we state the assumption for the operation of $\rho$ on a single column.}

\begin{restatable}{assumption}{rhotohdmx}
\label{assm:rhotohdmx}
For any $\zeta>0$ and $\eta \in (0, 1]$, $\exists~t > 0$ such that, for any column input $\vv$ satisfying $\evv_{j^*} - \max_{j\neq j^*} \evv_{j} \geq \zeta$ \textup{(}where $j^* = \argmax_j \evv_j$\textup{)}, we have $\rho[t \vv]_{j^*} \geq 1-\eta$ and $\sum_{j \neq j^*} \rho[t \vv]_j \leq \eta$.
\end{restatable}
Assumption~\ref{assm:rhotohdmx} requires that, for inputs that have some margin between the unique maximum entry and the other entries, $\rho[\cdot]$ can closely approximate the behavior of the hardmax operator
by scaling its input by a positive factor $t$. This assumption is satisfied by softmax $\sfmx$ and other sparse variants such as sparselin-gen and $\alpha$-entmax, as we show in \S~\ref{sec:proof-assm-rhotohdmx} of the supplementary material.

It is straightforward to check that the dense Transformer, which corresponds to $p = 1$, $\mc A_k^1 = [n]$, and $\rho[\cdot] = \sfmx[\cdot]$ in our framework, satisfies both Assumptions~\ref{assm:spsptrn} and \ref{assm:rhotohdmx}.

\vspace*{-3pt}
\subsection{Sparse Transformers are universal approximators}
\vspace*{-2pt}
\label{sec:thm}
The key justifying intuition for adopting sparse attention layers is that, if each token can attend to the other tokens in multiple hops\footnote{Note that this corresponds to our Assumption~\ref{assm:spsptrn}.\ref{assm:spsptrn-cond3}.}, then these models do not lose too much expressive power. 
{However, turning this intuition into a rigorous analysis is not straightforward. Moreover, recent results show that \textit{limited width} can render universal approximation \textit{impossible} even with arbitrary depth \citep{johnson2018deep, park2020minimum}, highlighting the challenges in analyzing sparse (limited ``width'') Transformers.}

We now state our main theorem, which shows that if the sparsity patterns $\{\mc A_k^l\}$ and the probability map $\rho$ satisfy Assumptions~\ref{assm:spsptrn} and \ref{assm:rhotohdmx}, sparse Transformers with $h=2$ attention heads of size $m=1$, and hidden layer width $r=4$ are universal approximators of continuous sequence-to-sequence functions on any compact domain (recall that $\mc F$ denotes the class of such continuous functions).
\begin{theorem}
\label{thm:main}
Consider any $f \in \mc F$, and the class of sparse Transformers $\mc {ST}^{2,1,4}$ \textup{(}cf.~\eqref{eq:transformer-fclass}\textup{)} with the underlying sparse attention layers satisfying Assumptions~\ref{assm:spsptrn} and \ref{assm:rhotohdmx}.
Then, for any $\epsilon >0$ and $1\leq p < \infty$, there exists a function $g \in \mc {ST}^{2,1,4}$ such that
\begin{equation*}
    \funcdist_p(f,g) \defeq \Big (\int_{\sD} \norm{f(\mX) - g(\mX)}_p^p d\mX \Big )^{1/p} \leq \epsilon.
\end{equation*}
\end{theorem}
\vspace*{-5pt}
{
As discussed earlier, dense Transformers satisfy Assumptions~\ref{assm:spsptrn} and \ref{assm:rhotohdmx}, which means that Theorem~\ref{thm:main} subsumes the existing result~\citep{yun2019transformers} for dense Transformers. 
We note that the required $h$, $m$, and $r$ in Theorem~\ref{thm:main} are independent of $d$, $n$, or the sparsity patterns.
We provide a high-level proof sketch of Theorem~\ref{thm:main} in \S~\ref{sec:main-proof-sketch}. 
There, we also discuss how many layers are sufficient for $\epsilon$-approximation of $f$, and show that Theorem~\ref{thm:main} requires only $p$ times more self-attention layers than \citet{yun2019transformers}.}


{
We would like to emphasize that Theorem~\ref{thm:main} provides the first \emph{formal evidence} that well-designed sparse attention layers do not limit Transformer's universal approximation power. 
In \S~\ref{sec:analyze-existing}, we show a surprising fact that some existing sparse self-attention layers with only $O(n)$ connections (as opposed to $n^2$ in regular self-attention layers) retain enough expressive power to approximate $\mc F$. Combined with the number of layers analyzed in \S~\ref{sec:main-proof-sketch}, this means that our analysis reduces the connections per layer from $n^2$ to $O(n)$, with only $p$ times more attention layers.}
This advantage of sparse Transformers over their dense counterpart becomes even stronger with increasing sequence length $n$, providing a theoretical support for the adoption of sparsity for tasks with long sequence lengths.

{On a final note, Theorem~\ref{thm:main} views the sequence length $n$ as a fixed constant. Hence, our result does not contradict a recent paper by \citet{hahn2020theoretical} which studies the limitation of Transformers for varying~$n$.
Also, our analysis applies to the encoder part of the Transformer network \citep{vaswani2017attention}. }

\vspace*{-3pt}
\subsection{Analysis of existing sparse Transformers}
\vspace*{-2pt}
\label{sec:analyze-existing}
By Theorem~\ref{thm:main}, any sparse Transformer that satisfies our Assumptions~\ref{assm:spsptrn} and \ref{assm:rhotohdmx} has universal approximation ability. In this section, we give some examples of such sparse Transformers.

\citet{child2019generating} propose two kinds of $2$-step sparsity patterns (i.e., $p = 2$) for sequence generation tasks, namely \textsc{Strided} and \textsc{Fixed} patterns. We consider the extension of their auto-regressive patterns (i.e., attending only to past tokens) to the whole sequence. In the \textsc{Strided} pattern, a token first attends to its $w$ neighbors and then attends to one token after every $w$ tokens in a strided manner. The sparsity pattern for the $k$-th token reads
\begin{equation}
\label{eq:stridedpattern}
\begin{aligned}
    \mc A^1_k &= [n] \cap \{k-\lceil \nicefrac{w}{2}\rceil, \dots, k-1, k, k+1, \dots, k+\lfloor \nicefrac{w}{2} \rfloor\},\\
    \mc A^2_k &= [n] \cap \{\dots, k-2w, k-w, k, k+w, k+2w, \dots \}.
\end{aligned}
\end{equation}
In the \textsc{Fixed} pattern, we divide the token into segments of length $w$. A token in a segment has access to other tokens in the same segment, and then the last tokens of the other segments:
\begin{equation}
\label{eq:fixedpattern}
\begin{aligned}
    \mc A^1_k &= [n] \cap \{\lceil \nicefrac{k}{w} \rceil \cdot w-w+1, \dots, \lceil \nicefrac{k}{w} \rceil \cdot w\},\quad\mc A^2_k = [n] \cap \left ( \{k\} \cup \{w, 2w, 3w, \dots \} \right ).
\end{aligned}
\end{equation}
The \textsc{Strided} and \textsc{Fixed} patterns satisfy both Assumption~\ref{assm:spsptrn} and \ref{assm:rhotohdmx} for all values of $w$. Specifically, Assumption~\ref{assm:spsptrn}.\ref{assm:spsptrn-cond3} holds with $s=2$, because any token can directly/indirectly access all the tokens in two hops. As for Assumption~\ref{assm:spsptrn}.\ref{assm:spsptrn-cond2}, the identity permutation $\gamma(i) = i$ suffices to satisfy the assumption for both patterns. By choosing $w = O(\sqrt{n})$, sparse Transformers with the \textsc{Strided} and \textsc{Fixed} patterns achieve universal approximation power with $O(n^{3/2})$ connections per attention layer.

\citet{guo2019star} consider the \textsc{Star} sparsity pattern where they add an auxiliary \emph{relay token} that attends to all the tokens, and the other tokens attend only to $2w$ neighboring tokens and the relay token. There is only one sparsity pattern, so $p=1$. The \textsc{Star} sparsity pattern can be written as
\begin{align}
\label{eq:starpattern}
    \mc A^1_k\!=\!\{n\} \cup \big \{(i-1)\text{ mod } (n-1)+1 \mid i \in \{k-w, \dots, k+w\} \big \} \text{ for } k \in [n-1],
    ~\mc A^1_n\!=\![n],
\end{align}
where $w \geq 1$.
For any fixed $w$, this sparse Transformer has $O(n)$ connections per attention layer, and it satisfies both assumptions. Specifically, Assumption~\ref{assm:spsptrn}.\ref{assm:spsptrn-cond2} is satisfied with the identity permutation, i.e., $\gamma(i) = (i)$ for $i \in [n]$. Since any token can access other tokens within two hops, Assumption~\ref{assm:spsptrn}.\ref{assm:spsptrn-cond3} is satisfied with $s = 2$. This demonstrates that $O(n)$ connections per layer suffice for sparse attention layers to have universal approximation power. One can similarly check that the sliding window sparsity patterns with/without global attention, proposed in Longformer \citep{beltagy2020longformer}, also satisfy the assumptions with $O(n)$ connections. 
{For the \textsc{BigBird} sparsity pattern~\citep{zaheer2020big}, it is also straightforward to check that a combination of its window attention and global attention satisfies Assumption~\ref{assm:spsptrn} with $O(n)$ connections.}
We state this interesting observation as a corollary below.
\begin{corollary}
\label{cor:O-of-n}
There exist sparse Transformers with $O(n)$ connections per self-attention layer that are universal approximators in the sense of Theorem~\ref{thm:main}.
\end{corollary}

Recall that another line of results that replaces softmax $\sfmx$ with sparse variants $\rho$ \citep{cui2019fine,correia2019adaptively,zhao2019explicit} also fits into our formulation, with  $p=1$ and $\mc A^1_k = [n]$.
As we show in \S~\ref{sec:proof-assm-rhotohdmx}, these alternative $\rho$'s satisfy Assumption~\ref{assm:rhotohdmx}. Thus, by Theorem~\ref{thm:main}, these models also have the universal approximation property.

\vspace*{-5pt}
\section{Proof sketch and discussion}
\vspace*{-5pt}
\subsection{Sketch of proof of Theorem~\ref{thm:main}}
\vspace*{-2pt}
\label{sec:main-proof-sketch}
Now, we sketch the proof of Theorem~\ref{thm:main}, which consists of three steps.
Throughout the proof, we assume without loss of generality that $\sD \subset [0,1)^{d \times n}$.

\noindent{\bf Step 1.}~
In the first step, we approximate $f \in \mc F$ with a piecewise constant function.
Towards this, consider a class of piecewise constant functions $\overline{\mc F}(\delta)$ that map $\sD$ to $\reals^{d \times n}$, where $\delta > 0$ and $\delta^{-1}$ is an integer. Any function in $\overline{\mc F}(\delta)$ maps cubes of the form $\mG+[0,\delta)^{d\times n}$ to matrices $\mA_\mG \in \reals^{d\times n}$, where $\mG \in \{ 0, \delta, \dots, 1-\delta \}^{d \times n}$.
We approximate $f$ with a function $\overline f \in \overline{\mc F}(\delta)$ such that $\funcdist_p(f,\overline f) \leq \epsilon/2$, by choosing small enough $\delta$. We defer the statement and the proof to \S~\ref{sec:proof-lem-main-step1} of the supplementary material.

\noindent{\bf Step 2.}~We then approximate $\overline f \in \overline{\mc F}(\delta)$ with a sparse Transformer network with a slightly modified architecture. In this architecture, we replace $\relu$ in the feed-forward layer with any piecewise linear activation $\phi \in \Phi$, where $\Phi$ denotes the class of {(possibly discontinuous)} piecewise linear functions with three pieces. We also replace $\rho$ in the sparse attention layer with the hardmax $\hdmx$ operator. We refer to the function class represented by the modified sparse Transformer as $\overline{\mc {ST}}^{h,m,r}$. 
{
By a careful construction, Lemma~\ref{lem:main-step2} shows that any $\overline f \in \overline{\mc F}(\delta)$ can be \emph{exactly} represented by the modified Transformer.
To this end, we first carefully choose the positional embedding $\mE$. We then quantize the inputs using feed-forward layers (Lemma~\ref{lem:quantize}), construct a \emph{contextual mapping} using self-attention layers to map the quantized inputs to unique ``ids'' (Lemma~\ref{lem:contextmap}), and then construct a \emph{value mapping} with feed-forward layers to map the ids to desired output values (Lemma~\ref{lem:valuemap}).}
See \S~\ref{sec:proof-lem-main-step2} and 
\S~\ref{sec:proof-lemmas} in the supplementary material for details.
\begin{lemma}
\label{lem:main-step2}
For any $\overline f \in \overline{\mc F}(\delta)$, there exists $\overline g \in \overline{\mc {ST}}^{2,1,1}$ such that $\overline{f}(\mX) = \overline{g}(\mX)$ for all $\mX \in \sD$.
\end{lemma}
\vspace*{-5pt}

\noindent{\bf Step 3.}~The final step is to approximate the function $\overline g \in \overline{\mc {ST}}^{2,1,1}$ with a sparse Transformer $g \in \mc{ST}^{2,1,4}$. This is done by approximating $\phi$ and $\hdmx$ with $\relu$ and $\rho$, respectively, while carefully bounding the accumulation of errors introduced by the approximation. See \S~\ref{sec:proof-lem-main-step3} in the supplementary material for the details.
\begin{lemma}
\label{lem:main-step3}
For $\overline g \in \overline{\mc {ST}}^{2,1,1}$ in Lemma~\ref{lem:main-step2}, there exists $g \in \mc {ST}^{2,1,4}$ such that $\funcdist_p(\overline g, g) \leq \epsilon/2$.
\end{lemma}
\vspace*{-5pt}

Combining these three steps, we establish that
  $\funcdist_p(f,g) \leq \funcdist_p(f,\overline f) + \funcdist_p(\overline f, \overline g) + \funcdist_p(\overline g,g) \leq \epsilon$.
\qed

{
\noindent{\bf How many layers are sufficient?}~
In \S~\ref{sec:proof-lem-main-step2}, Lemmas~\ref{lem:quantize}--\ref{lem:valuemap} show that we need $\frac{dn}{\delta}$ sparse Transformer blocks \eqref{eq:s-ff} for quantization, $\frac{p(n-1)}{\delta^d} + s$ for the contextual mapping, and $\frac{n}{\delta^{dn}}$~for the value mapping. Recall that $p$ is from~\eqref{eq:s-ff}, $s$ is from Assumption~\ref{assm:spsptrn}, and $\delta$ is from Step~1 above. 
In comparison, \S~C of \citep{yun2019transformers} shows that the dense counterpart requires $\frac{dn}{\delta}$, $\frac{n}{\delta^d} + 1$, and $\frac{n}{\delta^{dn}}$ Transformer blocks \eqref{eq:tb} for the three corresponding lemmas. Note two observations: 1) The value mapping \textbf{dominates} the depth, and its depth requirements are \textbf{identical} for the two cases; and 2) For contextual mappings (where the attention layers are used), we need roughly \textbf{$p$ times more} layers for sparse models. Recall from \S~\ref{sec:analyze-existing} that $p$ is usually a small constant. These observations mean that sparse Transformers can achieve universal approximation using depth of the \textbf{same order} in $d$, $n$ and $\delta$ as the dense Transformers.
}

\vspace*{-3pt}
\subsection{Key challenges in the proof}
\vspace*{-2pt}
\label{sec:main-proof-challenge}



While the high level outline of the proof is similar to the one for dense Transformers \citep{yun2019transformers}, the proof in \cite{yun2019transformers} crucially relies on having \emph{all} connections for computing attention in each layer, which we do not have in sparse Transformers.
{The sparsity in attention mechanism and the choice of general probability map $\rho$ pose nontrivial challenges in the proof.}
We highlight the key differences below.

Establishing the Step~2 of the dense result \citep{yun2019transformers} relies on constructing a \emph{contextual mapping} using attention layers. A contextual mapping is a function that maps tokens in different sequences to unique values, thereby allowing Transformers to distinguish the same token appearing in different contexts. 
A crucial ingredient in the construction of such a mapping is a shift operation implemented with two attention heads in an attention layer. This shift operation involves each token taking the maximum and minimum over the entire sequence, which obviously cannot be done with sparse Transformers as it would require each token to attend to all the other tokens in the sequence. We circumvent this issue by carefully choosing the positional embedding $\mE$ dependent on $\gamma$ (cf.~Assumption~\ref{assm:spsptrn}.\ref{assm:spsptrn-cond2}), and ensuring that a similar shift operation is applied in a desired order even under sparsity. 

As the final phase of the contextual mapping in \citep{yun2019transformers}, a single attention layer shifts the entire sequence by the maximum over the sequence. Again, this cannot be directly implemented due to sparsity. Using Assumption~\ref{assm:spsptrn}.\ref{assm:spsptrn-cond3}, we instead prove that by stacking $s$ sparse layers, one can successfully implement a similar operation that shifts the entire sequence by the maximum over the whole sequence, up to some controlled errors. This way, we overcome the difficulties posed by the sparsity and construct a new version of contextual mappings. The details can be found in \S~\ref{sec:proof-lem-contextmap} of the supplementary material.

Moreover, the proof of Step~3 in \citep{yun2019transformers} uses the simple fact that softmax can approximate hardmax arbitrarily closely. Since we do not restrict ourselves to softmax and generalize the probability map, a more careful argument is required. Since there are many layers in the network $\overline g$, it turns out that approximating it with an original sparse Transformer in $\mc {ST}^{2,1,4}$ requires carefully controlling the approximation errors accumulated over layers. The proof of Lemma~\ref{lem:main-step3} in \S~\ref{sec:proof-lem-main-step3} of the supplementary material shows that this is indeed possible by utilizing Assumption~\ref{assm:rhotohdmx}.

\vspace*{-5pt}
\section{Experiments}
\vspace*{-4pt}
\label{sec:exp}


We now present our experimental study comparing different design and implementation choices, including sparsity patterns and levels, on {four tasks: i)~a synthetic copying task, ii)~language modeling, iii)~translation, and iv)~GLUE tasks.} Our goal is to understand the effect of such choices while employing sparse Transformers to the tasks with small sequence lengths, complementing the existing results for sparse Transformers on long sequence tasks.

\vspace*{-3pt}
\subsection{Experiment Settings}
\vspace*{-2pt}
We consider four sparsity patterns: \textsc{Strided} \eqref{eq:stridedpattern}, \textsc{Fixed} \eqref{eq:fixedpattern}, \textsc{Star} \eqref{eq:starpattern} and \textsc{Random}. The first three patterns are proposed in \citep{child2019generating} and \citep{guo2019star}; we test them for different values of $w$. In case of the \textsc{Random} pattern, given a sparsity level, we make connections uniformly at random.
Following \citep{child2019generating}, \textsc{Strided} and \textsc{Fixed} patterns are tested for three different head configurations: i)~\textsc{Sequential}, where the sparse attention layers alternate between $\{\mc A_k^1\}$ and $\{\mc A_k^2\}$, as described in the previous sections; ii)~\textsc{Union}, where all sparse attention layers use the sparsity pattern $\{\mc A_k^1 \cup \mc A_k^2\}$; and iii)~\textsc{Multihead}, where half of the attention heads in every attention layer use $\{\mc A_k^1\}$ and the other half use $\{\mc A_k^2\}$. Note that, given the same sequence length, \textsc{Union} is less sparse than the other two configurations. Thus, to ensure fair comparisons, we compare different configurations based on their sparsity levels.

We use maximum sequence length 256 in all our {experiments, except 128 for GLUE tasks.} For the copying task, we experiment with only one sparse Transformer block (cf. Eq~\eqref{eq:s-ff}), with varying numbers of attention layers with $4$ attention heads. For language modeling and translation, we use the Tensor2Tensor \citep{vaswani2018tensor2tensor} framework and employ 12-block and 6-block (respectively) Transformers with $8$~attention heads per block. {For GLUE tasks, we experiment with the $\BB$ model.} For more details of the setup, see \S~\ref{sec:expsetup} of the supplementary material.

\vspace*{-3pt}
\subsection{Results}
\vspace*{-2pt}

\noindent{\bf Copying task.}~We consider a synthetic copying task proposed in \citep{kitaev2020reformer}, where the input sequence has the format $0\mathbf{s}0\mathbf{s}$, where $\mathbf{s}$ is a 127 length sequence of symbols in $[0, 127]$. The models have to predict (copy) the second part, given the first half of the input. This task tests the ability of sparse Transformers to communicate the information. Table~\ref{tab:synthetic} presents the results for this task. 
Except for the \textsc{Star} and \textsc{Random} patterns, we can see that the networks learn to copy the sequences with four sparse attention layers. One possible explanation for the bad performance of \textsc{Star} is that, except for the relay token, it only attends to local neighbors while the task requires to copy distant tokens.

\noindent{\bf Language modeling.}~We conduct the language modeling experiments on the One Billion Word Benchmark \citep{lm1b} which has almost one billion tokens and a vocabulary of more than 800K unique tokens. In Figure~\ref{fig:lm1b}, we plot the perplexity against the sparsity level. We observe that the \textsc{Strided} pattern and the \textsc{Star} achieve the best performance across all sparsity levels. For both the \textsc{Strided} and \textsc{Fixed} patterns, the \textsc{Union} configuration shows the best performance.


\begin{table*}[t!]
    \centering
    \caption{Accuracy on the synthetic copying task. Percentages in parentheses mark the sparsity levels.}
    \vspace{0.5\baselineskip}
    \setlength{\tabcolsep}{3pt}
    \scalebox{0.9}{
    \begin{tabular}{ccccccccc}
    \toprule
    \multicolumn{1}{l}{} &
    \multicolumn{3}{c}{{\sc Strided}}   &
    \multicolumn{3}{c}{{\sc Fixed}}  & {\sc Star} & {\sc Random}\\
    \cmidrule(lll){2-4} \cmidrule(lll){5-7}
        {\bf Depth}
        & \begin{tabular}{@{}c@{}}{\sc Union} \\ (87\%)\end{tabular}
        & \begin{tabular}{@{}c@{}}{\sc Multihead} \\ (93\%)\end{tabular} & \begin{tabular}{@{}c@{}}{\sc Sequential} \\ (93\%)\end{tabular} &\begin{tabular}{@{}c@{}}{\sc Union} \\ (87\%)\end{tabular}
        & \begin{tabular}{@{}c@{}}{\sc Multihead} \\ (93\%)\end{tabular} & \begin{tabular}{@{}c@{}}{\sc Sequential} \\ (93\%)\end{tabular} & (87\%) & (90\%)\\
        \midrule 
        1-layer &  0.82\% & 0.82\% & 0.80\% & 7.04\% & 0.76\% & 0.80\% & 1.53\% & 33.14\% \\
        2-layer &  100.00\% & 100.00\% & 81.24\% & 69.26\% & 56.45\% & 96.01\% & 29.70\% & 63.41\% \\
        3-layer &  100.00\% & 100.00\% & 100.00\% & 99.98\% & 99.08\% & 98.58\% & 42.18\% & 70.29\% \\
        4-layer &  100.00\% & 100.00\% & 100.00\% & 100.00\% & 99.64\% & 100.00\% & 83.57\% & 95.49\% \\
        \bottomrule
    \end{tabular}
    }
    \label{tab:synthetic}
    \vspace{-0.5\baselineskip}
\end{table*}

\begin{figure}[t]
    \centering
    \subfloat[One Billion Benchmark]{
    \label{fig:lm1b}
      \includegraphics[width=0.49\textwidth]{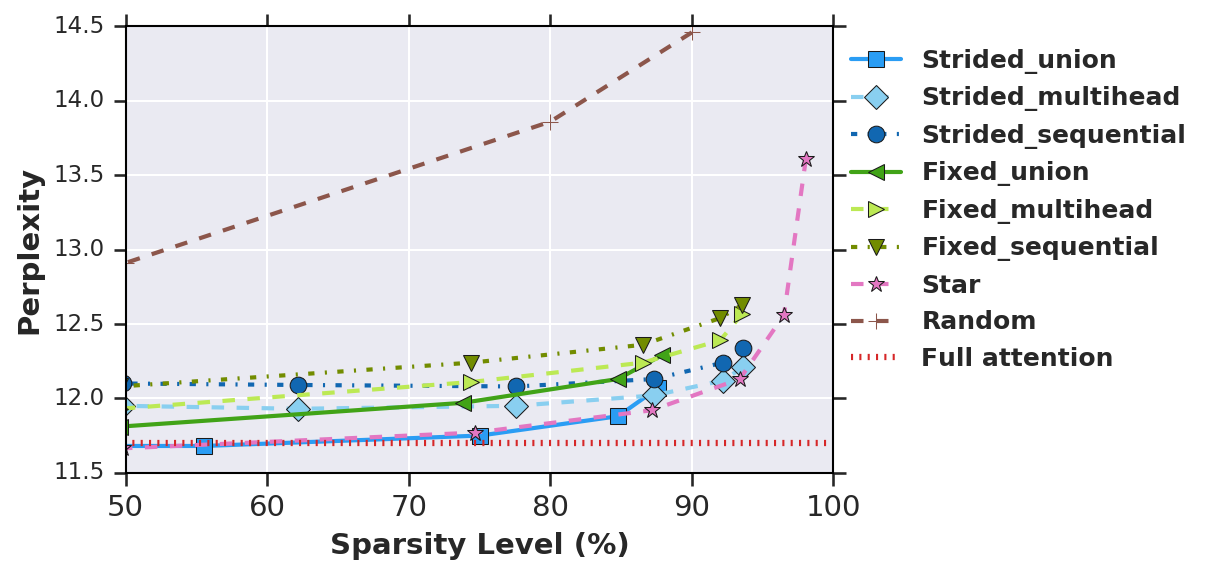}
    }
    \subfloat[WMT en-cs]{
    \label{fig:translate_encs}
        \includegraphics[width=0.49\textwidth]{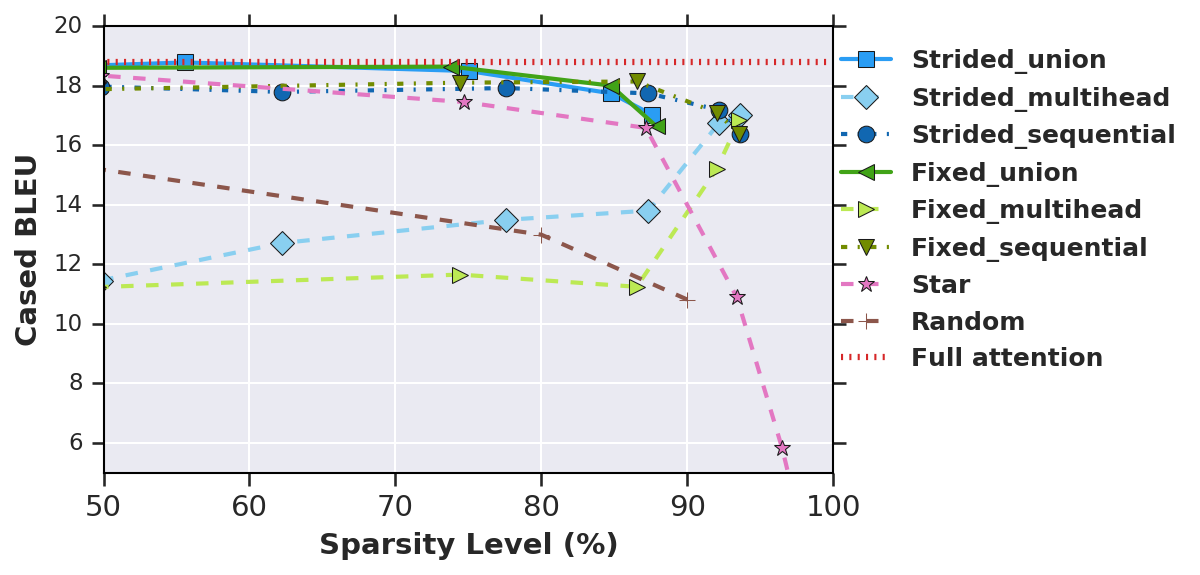}
    }
    \caption{Comparison of sparsity patterns and different head configurations on the One Billion Benchmark (a language modeling task) and WMT en-cs (a translation task). Note that the number of connections in the attention layers goes down as we increase the sparsity level.}
\end{figure}

\begin{figure}[t]
    \centering
    \subfloat[MNLI]{
    \label{fig:bert_mnli}
      \includegraphics[width=0.49\textwidth]{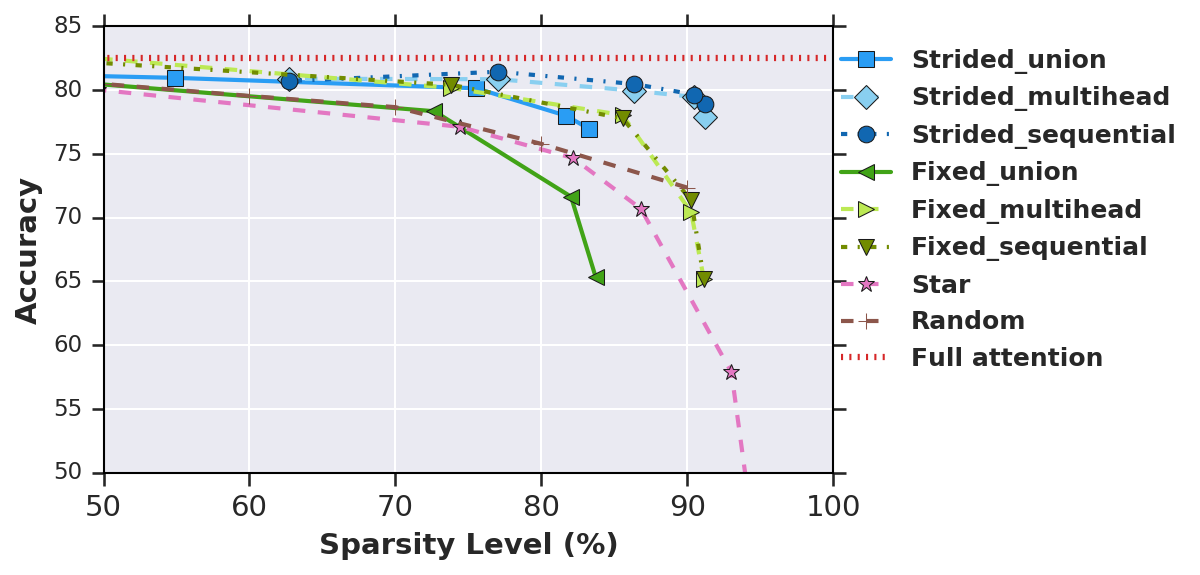}
    }
    \subfloat[XNLI]{
    \label{fig:bert_xnli}
        \includegraphics[width=0.49\textwidth]{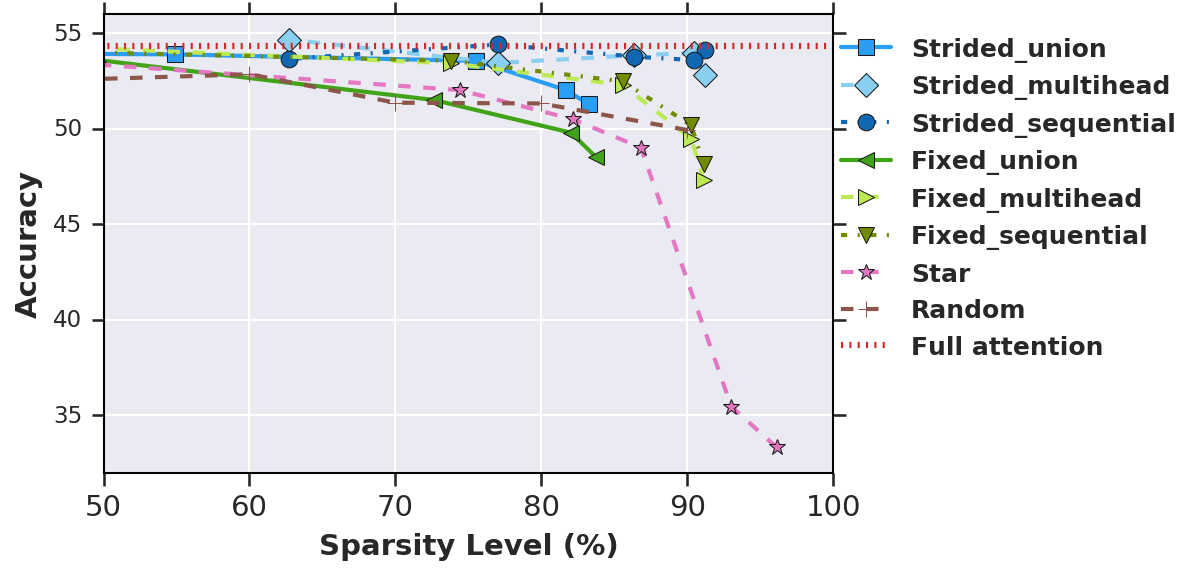}
    }
    \caption{{Comparison of sparsity patterns and different head configurations on the MNLI and XNLI (sentence-pair classification tasks), using the $\BB$ model.}}
\end{figure}

\noindent{\bf Translation.}~For the translation task, we train the model on WMT18 English-Czech (en-cs) dataset and test it on the Newstest 2015 dataset. We plot the BLEU score against the sparsity level in Figure~\ref{fig:translate_encs}. We apply the same sparsity pattern to both the encoder and the decoder. 
The \textsc{Strided} and \textsc{Fixed} patterns with \textsc{Union} configuration show the best scores, which are similar to the dense attention. The \textsc{Union} configuration is also the least sensitive to the sparsity levels. 


\noindent{\bf GLUE Tasks.}~{We experiment with the $\BB$ model and report results on two sentence-pair classification tasks: MNLI \citep{williams2018broad} (Figure~\ref{fig:bert_mnli}) and XNLI \citep{conneau2018xnli} (Figure~\ref{fig:bert_xnli}). We plot the average accuracy of three runs on the dev set against the sparsity level. Additional results of the CoLA and MRPC tasks are reported in \S~\ref{sec:additional_exp} of the supplementary material.  }

\vspace*{-8pt}
\paragraph{Discussion.} 
{
In all tasks, the \textsc{Random} pattern performs worse than the deterministic patterns, demonstrating the need for a careful design of sparsity patterns.}
Overall, our experiments suggest that the design of the optimal sparsity patterns is {heavily} dependent on specific tasks. For example, the \textsc{Star} pattern shows the best performance on the language modeling task, {while having trouble with copying, translation, and BERT experiments.} Among the three head configurations tested for \textsc{Strided} and \textsc{Fixed}, {the \textsc{Union} performs the best in language modeling and translation but suffers in BERT tasks. In translation experiments, we see an interesting trend that the performance of \textsc{Multihead} configuration improves as sparsity increases. We conjecture that this is due to the fact that in \textsc{Strided} and \textsc{Fixed}, we have $|\mc A_k^1| = O(w)$ and $|\mc A_k^2| = O(\nicefrac{n}{w})$ (cf. Eqs~\eqref{eq:stridedpattern} and \eqref{eq:fixedpattern}), so the sparsest choice of $w = O(\sqrt{n})$ is the one with the best ``balance'' between $|\mc A_k^1|$ and $|\mc A_k^2|$. }


\vspace*{-5pt}
\section{Conclusion}
\vspace*{-4pt}
Recently, sparse Transformers have received a lot of attention as they enable more efficient/faster attention mechanisms for the tasks with very long sequence lengths. We take an initial step to provide a theoretical understanding of these models.
We provide a unifying framework that captures existing sparse attention models, and prove a universal approximation theorem for sparse Transformers which holds under intuitive conditions on sparsity patterns and probability maps. We also carry out experiments comparing different sparsity patterns and levels on standard NLP tasks. We hope that this work will shed light on the understanding of sparsity in attention layers, and provide guidance for the design of sparse attention models.

\section*{Broader Impact}
This work studies theoretical aspects of a class of widely used neural network models in NLP and related areas.
Since we do not propose a new method nor a new dataset, we expect that the impact of this work on ethical aspects and future societal consequences will be small, if any.
Other than that, this work brings new insights into the sparsity in attention models, hence may make an impact on the study of faster and more efficient NLP models.

\begin{ack}
{
CY acknowledges partial support as a graduate Research Assistant from the NSF Grant (CAREER 1846088). CY also acknowledges Korea Foundation for Advanced Studies for their support.
}
\end{ack}



\bibliography{cite}

\begin{thebibliography}{36}
\providecommand{\natexlab}[1]{#1}
\providecommand{\url}[1]{\texttt{#1}}
\expandafter\ifx\csname urlstyle\endcsname\relax
  \providecommand{\doi}[1]{doi: #1}\else
  \providecommand{\doi}{doi: \begingroup \urlstyle{rm}\Url}\fi

\bibitem[Bahdanau et~al.(2015)Bahdanau, Cho, and Bengio]{bahdanau2015neural}
Dzmitry Bahdanau, Kyunghyun Cho, and Yoshua Bengio.
\newblock Neural machine translation by jointly learning to align and
  translate.
\newblock In \emph{International Conference on Learning Representations}, 2015.

\bibitem[Beltagy et~al.(2020)Beltagy, Peters, and Cohan]{beltagy2020longformer}
Iz~Beltagy, Matthew~E Peters, and Arman Cohan.
\newblock Longformer: The long-document {T}ransformer.
\newblock \emph{arXiv preprint arXiv:2004.05150}, 2020.

\bibitem[Bhojanapalli et~al.(2020)Bhojanapalli, Yun, Rawat, Reddi, and
  Kumar]{bhojanapalli2020low}
Srinadh Bhojanapalli, Chulhee Yun, Ankit~Singh Rawat, Sashank~J Reddi, and
  Sanjiv Kumar.
\newblock Low-rank bottleneck in multi-head attention models.
\newblock \emph{arXiv preprint arXiv:2002.07028}, 2020.

\bibitem[Brunner et~al.(2019)Brunner, Liu, Pascual, Richter, Ciaramita, and
  Wattenhofer]{brunner2019identifiability}
Gino Brunner, Yang Liu, Damián Pascual, Oliver Richter, Massimiliano
  Ciaramita, and Roger Wattenhofer.
\newblock On identifiability in {T}ransformers.
\newblock \emph{arXiv preprint arXiv:1908.04211}, 2019.

\bibitem[Chelba et~al.(2013)Chelba, Mikolov, Schuster, Ge, Brants, and
  Koehn]{lm1b}
Ciprian Chelba, Tomas Mikolov, Mike Schuster, Qi~Ge, Thorsten Brants, and
  Phillipp Koehn.
\newblock One billion word benchmark for measuring progress in statistical
  language modeling.
\newblock \emph{CoRR}, abs/1312.3005, 2013.
\newblock URL \url{http://arxiv.org/abs/1312.3005}.

\bibitem[Child et~al.(2019)Child, Gray, Radford, and
  Sutskever]{child2019generating}
Rewon Child, Scott Gray, Alec Radford, and Ilya Sutskever.
\newblock Generating long sequences with sparse {T}ransformers.
\newblock \emph{arXiv preprint arXiv:1904.10509}, 2019.

\bibitem[Conneau et~al.(2018)Conneau, Rinott, Lample, Williams, Bowman,
  Schwenk, and Stoyanov]{conneau2018xnli}
Alexis Conneau, Ruty Rinott, Guillaume Lample, Adina Williams, Samuel~R.
  Bowman, Holger Schwenk, and Veselin Stoyanov.
\newblock Xnli: Evaluating cross-lingual sentence representations.
\newblock In \emph{Proceedings of the 2018 Conference on Empirical Methods in
  Natural Language Processing}. Association for Computational Linguistics,
  2018.

\bibitem[Correia et~al.(2019)Correia, Niculae, and
  Martins]{correia2019adaptively}
Gon{\c{c}}alo~M Correia, Vlad Niculae, and Andr{\'e}~FT Martins.
\newblock Adaptively sparse {T}ransformers.
\newblock \emph{arXiv preprint arXiv:1909.00015}, 2019.

\bibitem[Cui et~al.(2019)Cui, Li, Chen, and Zhang]{cui2019fine}
Baiyun Cui, Yingming Li, Ming Chen, and Zhongfei Zhang.
\newblock Fine-tune {BERT} with sparse self-attention mechanism.
\newblock In \emph{Proceedings of the 2019 Conference on Empirical Methods in
  Natural Language Processing and the 9th International Joint Conference on
  Natural Language Processing (EMNLP-IJCNLP)}, pages 3539--3544, 2019.

\bibitem[Devlin et~al.(2018)Devlin, Chang, Lee, and Toutanova]{devlin2018bert}
Jacob Devlin, Ming-Wei Chang, Kenton Lee, and Kristina Toutanova.
\newblock {BERT}: Pre-training of deep bidirectional {T}ransformers for
  language understanding.
\newblock \emph{arXiv preprint arXiv:1810.04805}, 2018.

\bibitem[Guo et~al.(2019)Guo, Qiu, Liu, Shao, Xue, and Zhang]{guo2019star}
Qipeng Guo, Xipeng Qiu, Pengfei Liu, Yunfan Shao, Xiangyang Xue, and Zheng
  Zhang.
\newblock Star-{T}ransformer.
\newblock \emph{arXiv preprint arXiv:1902.09113}, 2019.

\bibitem[Hahn(2020)]{hahn2020theoretical}
Michael Hahn.
\newblock Theoretical limitations of self-attention in neural sequence models.
\newblock \emph{Transactions of the Association for Computational Linguistics},
  8:\penalty0 156--171, 2020.

\bibitem[Johnson(2019)]{johnson2018deep}
Jesse Johnson.
\newblock Deep, skinny neural networks are not universal approximators.
\newblock In \emph{International Conference on Learning Representations}, 2019.

\bibitem[Kitaev et~al.(2020)Kitaev, Kaiser, and Levskaya]{kitaev2020reformer}
Nikita Kitaev, {\L}ukasz Kaiser, and Anselm Levskaya.
\newblock Reformer: The efficient {T}ransformer.
\newblock \emph{arXiv preprint arXiv:2001.04451}, 2020.

\bibitem[Laha et~al.(2018)Laha, Chemmengath, Agrawal, Khapra, Sankaranarayanan,
  and Ramaswamy]{laha2018controllable}
Anirban Laha, Saneem~Ahmed Chemmengath, Priyanka Agrawal, Mitesh Khapra,
  Karthik Sankaranarayanan, and Harish~G Ramaswamy.
\newblock On controllable sparse alternatives to softmax.
\newblock In \emph{Advances in Neural Information Processing Systems}, pages
  6422--6432, 2018.

\bibitem[Li et~al.(2020)Li, Meng, Han, Wu, and Li]{li2020sac}
Xiaoya Li, Yuxian Meng, Qinghong Han, Fei Wu, and Jiwei Li.
\newblock Sac: Accelerating and structuring self-attention via sparse adaptive
  connection.
\newblock \emph{arXiv preprint arXiv:2003.09833}, 2020.

\bibitem[Liu et~al.(2019)Liu, Ott, Goyal, Du, Joshi, Chen, Levy, Lewis,
  Zettlemoyer, and Stoyanov]{roberta2019}
Yinhan Liu, Myle Ott, Naman Goyal, Jingfei Du, Mandar Joshi, Danqi Chen, Omer
  Levy, Mike Lewis, Luke Zettlemoyer, and Veselin Stoyanov.
\newblock {RoBERTa}: {A} robustly optimized {BERT} pretraining approach.
\newblock \emph{arXiv preprint arXiv:1907.11692}, 2019.

\bibitem[Luong et~al.(2015)Luong, Pham, and Manning]{luong2015multiplicativ}
Minh-Thang Luong, Hieu Pham, and Christopher~D. Manning.
\newblock Effective approaches to attention-based neural machine translation.
\newblock In \emph{Empirical Methods in Natural Language Processing (EMNLP)},
  pages 1412--1421, Lisbon, Portugal, September 2015. Association for
  Computational Linguistics.

\bibitem[Park et~al.(2020)Park, Yun, Lee, and Shin]{park2020minimum}
Sejun Park, Chulhee Yun, Jaeho Lee, and Jinwoo Shin.
\newblock Minimum width for universal approximation.
\newblock \emph{arXiv preprint arXiv:2006.08859}, 2020.

\bibitem[P{\'e}rez et~al.(2019)P{\'e}rez, Marinkovi{\'c}, and
  Barcel{\'o}]{perez2019turing}
Jorge P{\'e}rez, Javier Marinkovi{\'c}, and Pablo Barcel{\'o}.
\newblock On the {Turing} completeness of modern neural network architectures.
\newblock \emph{arXiv preprint arXiv:1901.03429}, 2019.

\bibitem[Peters et~al.(2019)Peters, Niculae, and Martins]{peters2019sparse}
Ben Peters, Vlad Niculae, and Andr{\'e}~FT Martins.
\newblock Sparse sequence-to-sequence models.
\newblock \emph{arXiv preprint arXiv:1905.05702}, 2019.

\bibitem[Qiu et~al.(2019)Qiu, Ma, Levy, Yih, Wang, and Tang]{qiu2019blockwise}
Jiezhong Qiu, Hao Ma, Omer Levy, Scott Wen-tau Yih, Sinong Wang, and Jie Tang.
\newblock Blockwise self-attention for long document understanding.
\newblock \emph{arXiv preprint arXiv:1911.02972}, 2019.

\bibitem[Radford et~al.(2018)Radford, Narasimhan, Salimans, and
  Sutskever]{radford2018gpt}
Alec Radford, Karthik Narasimhan, Tim Salimans, and Ilya Sutskever.
\newblock Improving language understanding by generative pre-training.
\newblock \emph{Technical Report, {O}pen{AI}}, 2018.

\bibitem[Radford et~al.(2019)Radford, Wu, Child, Luan, Amodei, and
  Sutskever]{radford2019gpt2}
Alec Radford, Jeffrey Wu, Rewon Child, David Luan, Dario Amodei, and Ilya
  Sutskever.
\newblock Language models are unsupervised multitask learners.
\newblock \emph{Technical Report, {O}pen{AI}}, 2019.

\bibitem[Roy et~al.(2020)Roy, Saffar, Vaswani, and Grangier]{roy2020efficient}
Aurko Roy, Mohammad Saffar, Ashish Vaswani, and David Grangier.
\newblock Efficient content-based sparse attention with routing {T}ransformers.
\newblock \emph{arXiv preprint arXiv:2003.05997}, 2020.

\bibitem[Sukhbaatar et~al.(2019)Sukhbaatar, Grave, Bojanowski, and
  Joulin]{sukhbaatar2019adaptive}
Sainbayar Sukhbaatar, Edouard Grave, Piotr Bojanowski, and Armand Joulin.
\newblock Adaptive attention span in {T}ransformers.
\newblock \emph{arXiv preprint arXiv:1905.07799}, 2019.

\bibitem[Tay et~al.(2020)Tay, Dehghani, Bahri, and Metzler]{tay2020efficient}
Yi~Tay, Mostafa Dehghani, Dara Bahri, and Donald Metzler.
\newblock Efficient transformers: A survey.
\newblock \emph{arXiv preprint arXiv:2009.06732}, 2020.

\bibitem[Vaswani et~al.(2017)Vaswani, Shazeer, Parmar, Uszkoreit, Jones, Gomez,
  Kaiser, and Polosukhin]{vaswani2017attention}
Ashish Vaswani, Noam Shazeer, Niki Parmar, Jakob Uszkoreit, Llion Jones,
  Aidan~N Gomez, {\L}ukasz Kaiser, and Illia Polosukhin.
\newblock Attention is all you need.
\newblock In \emph{Advances in {N}eural {I}nformation {P}rocessing {S}ystems},
  pages 5998--6008, 2017.

\bibitem[Vaswani et~al.(2018)Vaswani, Bengio, Brevdo, Chollet, Gomez, Gouws,
  Jones, Kaiser, Kalchbrenner, Parmar, et~al.]{vaswani2018tensor2tensor}
Ashish Vaswani, Samy Bengio, Eugene Brevdo, Francois Chollet, Aidan~N Gomez,
  Stephan Gouws, Llion Jones, {\L}ukasz Kaiser, Nal Kalchbrenner, Niki Parmar,
  et~al.
\newblock Tensor2tensor for neural machine translation.
\newblock \emph{arXiv preprint arXiv:1803.07416}, 2018.

\bibitem[Williams et~al.(2018)Williams, Nangia, and Bowman]{williams2018broad}
Adina Williams, Nikita Nangia, and Samuel Bowman.
\newblock A broad-coverage challenge corpus for sentence understanding through
  inference.
\newblock In \emph{Proceedings of the 2018 Conference of the North American
  Chapter of the Association for Computational Linguistics: Human Language
  Technologies, Volume 1 (Long Papers)}, pages 1112--1122. Association for
  Computational Linguistics, 2018.
\newblock URL \url{http://aclweb.org/anthology/N18-1101}.

\bibitem[Yang et~al.(2019)Yang, Dai, Yang, Carbonell, Salakhutdinov, and
  Le]{xlnet2019}
Zhilin Yang, Zihang Dai, Yiming Yang, Jaime~G. Carbonell, Ruslan Salakhutdinov,
  and Quoc~V. Le.
\newblock {XLNet}: Generalized autoregressive pretraining for language
  understanding.
\newblock \emph{arXiv preprint arXiv:1906.08237}, 2019.

\bibitem[Ye et~al.(2019)Ye, Guo, Gan, Qiu, and Zhang]{ye2019bp}
Zihao Ye, Qipeng Guo, Quan Gan, Xipeng Qiu, and Zheng Zhang.
\newblock Bp-{T}ransformer: Modelling long-range context via binary
  partitioning.
\newblock \emph{arXiv preprint arXiv:1911.04070}, 2019.

\bibitem[Yun et~al.(2020)Yun, Bhojanapalli, Rawat, Reddi, and
  Kumar]{yun2019transformers}
Chulhee Yun, Srinadh Bhojanapalli, Ankit~Singh Rawat, Sashank~J Reddi, and
  Sanjiv Kumar.
\newblock Are {T}ransformers universal approximators of sequence-to-sequence
  functions?
\newblock In \emph{International Conference on Learning Representations}, 2020.

\bibitem[Zaheer et~al.(2020)Zaheer, Guruganesh, Dubey, Ainslie, Alberti,
  Ontanon, Pham, Ravula, Wang, Yang, et~al.]{zaheer2020big}
Manzil Zaheer, Guru Guruganesh, Avinava Dubey, Joshua Ainslie, Chris Alberti,
  Santiago Ontanon, Philip Pham, Anirudh Ravula, Qifan Wang, Li~Yang, et~al.
\newblock Big bird: Transformers for longer sequences.
\newblock \emph{arXiv preprint arXiv:2007.14062}, 2020.

\bibitem[Zhao et~al.(2019)Zhao, Lin, Zhang, Ren, Su, and Sun]{zhao2019explicit}
Guangxiang Zhao, Junyang Lin, Zhiyuan Zhang, Xuancheng Ren, Qi~Su, and Xu~Sun.
\newblock Explicit sparse {T}ransformer: Concentrated attention through
  explicit selection.
\newblock \emph{arXiv preprint arXiv:1912.11637}, 2019.

\bibitem[Zhu et~al.(2015)Zhu, Kiros, Zemel, Salakhutdinov, Urtasun, Torralba,
  and Fidler]{zhu2015aligning}
Yukun Zhu, Ryan Kiros, Rich Zemel, Ruslan Salakhutdinov, Raquel Urtasun,
  Antonio Torralba, and Sanja Fidler.
\newblock Aligning books and movies: Towards story-like visual explanations by
  watching movies and reading books.
\newblock In \emph{Proceedings of the IEEE international conference on computer
  vision}, pages 19--27, 2015.

\end{thebibliography}
\bibliographystyle{plainnat}

\newpage
\appendix

\section{Outline and notation}
The supplementary material is organized as follows. First, \S~\ref{sec:proof-assm-rhotohdmx} proves that the softmax operator as well as its sparse versions indeed satisfy Assumption~\ref{assm:rhotohdmx}. Next, \S~\ref{sec:proof-lem-main-step1} provides formal statements of Step~1 in the proof sketch (\S~\ref{sec:main-proof-sketch}). The outline of proof of Lemma~\ref{lem:main-step2} (Step~2 in the proof sketch) is presented in \S~\ref{sec:proof-lem-main-step2}, followed by a separate section (\S~\ref{sec:proof-lemmas}) proving the three key sublemmas in the proof. The proof of Step~3, Lemma~\ref{lem:main-step3}, is given in \S~\ref{sec:proof-lem-main-step3}. {Lastly, \S~\ref{sec:expsetup} and \S~\ref{sec:additional_exp} present the detailed setup of our experiments and additional experiment results, respectively.}

We next review some of the notation and also introduce additional notation used throughout the supplementary material.
For a positive integer $a$, let $[a] \defeq \{1, \dots, a\}$.
For $a, b, c \in \reals$ where $b-a > 0$ is an integer multiple of $c > 0$, we write $[a:c:b] \defeq \{a, a+c, a+2c, \dots, b-c, b\}$.
For any matrix $\mA \in \reals^{d \times n}$, let $\mA_j$ denote its $j$-th column, and $\mA_{\mc S}$ denote the submatrix consisting of columns of $\mA$ in the index set $\mc S \subseteq [n]$. We also use $\emA_{i,j}$ to denote its $(i,j)$-th entry. 
Let $\indic{\cdot}$ be the 0-1 indicator for an event.
Let $\ones_n \in \reals^n$ be a vector whose components are all 1.

\section{Sparse probability maps satisfy Assumption~\ref{assm:rhotohdmx}}
\label{sec:proof-assm-rhotohdmx}
In this section, we show that the softmax operator $\sfmx$ as well as the probability maps $\rho$ used to replace softmax in the existing approaches, namely softmax with only top-$k$ inputs \citep{zhao2019explicit}, sparselin-gen \citep{cui2019fine}, and $\alpha$-entmax \citep{correia2019adaptively}, all satisfy Assumption~\ref{assm:rhotohdmx}. We restate the assumption for reader's convenience:
\rhotohdmx*
As in the assumption, we only consider the operation of these probability maps on a single vector, as they are applied column-wise. For each of the probability maps, we will show that for any $\zeta>0$ and $\eta \in (0,1]$, we can choose $t>0$ that satisfies the conditions of Assumption~\ref{assm:rhotohdmx}.

\subsection{Softmax \& softmax with top-$k$ inputs}
Given an input vector $\vv \in \reals^n$, the $j$-th coordinate of the output of softmax $\sfmx[\vv]$ is defined as
\begin{equation*}
    \sfmx[\vv]_j \defeq \frac{\exp(\evv_j)}{\sum_{i=1}^n \exp(\evv_i)}.
\end{equation*}
We assume without loss of generality that the entry of $\vv$ is in decreasing order, where the first two entries satisfy $\evv_1 - \evv_2 \geq \zeta$.
For any such $\zeta > 0$ and any $0<\eta\leq 1$, our aim is to show the existence of $t > 0$ such that 
    $\sfmx[t \vv]_1 = \frac{\exp(t\evv_1)}{\sum_{i=1}^n \exp(t \evv_i)} \geq 1-\eta$.
Then, $\sum_{j=2}^n \sfmx[t \vv]_j \leq \eta$ follows.

Now, since $\evv_i \leq \evv_1 - \zeta$ for $i \in [2:n]$, note that
\begin{align*}
    \sfmx[t \vv]_1 
    = \frac{\exp(t\evv_1)}{\sum_{i=1}^n \exp(t \evv_i)}
    \geq \frac{\exp(t\evv_1)}{\exp(t\evv_1) + (n-1) \exp(t \evv_1-t \zeta)}
    = \frac{1}{1 + (n-1) \exp(-t \zeta)}.
\end{align*}
Since $\frac{1}{1 + (n-1) \exp(-t \zeta)}$ is an increasing function in $t>0$, one can increase $t$ sufficiently large to make it greater than $1-\eta$.

The same argument holds for the softmax with top-$k$ inputs, used in \citep{zhao2019explicit}. By the assumption on $\vv$, entries $\evv_1, \dots, \evv_k$ are the top $k$ components. Thus, 
\begin{align*}
    \rho[t \vv]_1 \geq \frac{1}{1 + (k-1) \exp(-t \zeta)} \geq 1-\eta
\end{align*}
can be satisfied by choosing large enough $t > 0$.

\subsection{Sparselin-gen}
We now consider the case where $\rho$ is sparselin-gen \citep{laha2018controllable}, which was used to sparsify the attention score matrices in \citep{cui2019fine}. 
Given a regularization parameter $\lambda \in [0, 1)$, the sparselin-gen used in \citep{cui2019fine} is defined as
\begin{equation*}
    \rho[\vv] \defeq \argmin_{\vp \in \Delta^{n-1}} \norm{\vp-\vv}^2 - \lambda \norm{\vp}^2,
\end{equation*}
where $\Delta^{n-1}\defeq \{ \vp \in \reals^n \mid \vp \geq \zeros, \sum_{i=1}^n \evp_i = 1\}$ is the probability simplex. Then, the solution for optimization problem above can be written as
\begin{equation*}
    \rho[\vv]_j = \max \left \{ 0, \frac{\evv_j - \tau(\vv)}{1-\lambda} \right \}, \text{ for } j \in [n],
\end{equation*}
where $\tau : \reals^n \to \reals$ is a threshold function that chooses the threshold $\tau(\vv)$ such that $\sum_{j=1}^n \rho[\vv]_j = 1$.

Now, assume without loss of generality that the entry of $\vv$ is in decreasing order, where the first two entries satisfy $\evv_1 - \evv_2 \geq \zeta$.
For any such $\zeta > 0$ and any $0<\eta\leq 1$, our aim is to show the existence of $t > 0$ such that $\rho[t \vv]_1 \geq 1-\eta$.
This is done by choosing $t = \frac{1-\eta}{\zeta}$.
To see this, notice that if $\evv_j$'s are in decreasing order, then $\rho[\vv]_j$ are also in decreasing order. Now consider
\begin{equation*}
    \rho[t\vv]_1 = \max \left \{ 0, \frac{t\evv_1 - \tau(t\vv)}{1-\lambda} \right \},~~
    \rho[t\vv]_2 = \max \left \{ 0, \frac{t\evv_2 - \tau(t\vv)}{1-\lambda} \right \}.
\end{equation*}
If $\rho[t\vv]_2 = 0$, then $\rho[t\vv]_j = 0$ for all $j = 3, \dots, n$, and $\rho[t\vv]_1 = 1 \geq 1- \eta$.
If $\rho[t\vv]_2 > 0$, then
\begin{equation*}
    \rho[t\vv]_1 - \rho[t\vv]_2 = \frac{t\evv_1 - \tau(t\vv)}{1-\lambda} - \frac{t\evv_2 - \tau(t\vv)}{1-\lambda} = \frac{t (\evv_1 - \evv_2)}{1-\lambda} \geq t (\evv_1 - \evv_2) \geq t \zeta = 1 - \eta.
\end{equation*}

\subsection{$\alpha$-entmax}
Next, we consider the case where $\rho$ is $\alpha$-entmax \citep{peters2019sparse}, which was used to sparsify the attention score matrices in \citep{correia2019adaptively}. Given a parameter $\alpha \geq 1$, the $\alpha$-entmax is defined as 
\begin{equation*}
    \rho[\vv] \defeq \argmax_{\vp \in \Delta^{n-1}} \vp^T \vv + H_\alpha(\vv),
\end{equation*}
where $\Delta^{n-1}$ is the probability simplex and $H_\alpha$ is the Tsallis continuous family of entropies
\begin{equation*}
    H_\alpha(\vv) \defeq 
    \begin{cases}
    \frac{1}{\alpha(\alpha-1)} \sum_j \evv_j - \evv_j^\alpha & \alpha > 1,\\
    -\sum_j \evv_j \log \evv_j & \alpha = 1.
    \end{cases}
\end{equation*}
As shown in \citep{correia2019adaptively}, the solution of $\alpha$-entmax is equal to softmax if $\alpha = 1$, and otherwise ($\alpha > 1$) it is given in the form
\begin{equation*}
    \rho[\vv]_j = \big [\max\{ 0, (\alpha-1)\evv_j - \tau(\vv) \} \big]^{\frac{1}{\alpha-1}}, \text{ for } j \in [n],
\end{equation*}
where $\tau : \reals^n \to \reals$ is a threshold function that chooses the threshold $\tau(\vv)$ such that $\sum_{j=1}^n \rho [\vv]_j = 1$. Since softmax ($\alpha=1$) is already covered above, we focus on $\alpha > 1$.

Again, assume without loss of generality that the entry of $\vv$ is in decreasing order, where the first two entries satisfy $\evv_1 - \evv_2 \geq \zeta$.
For any such $\zeta > 0$ and any $0<\eta\leq 1$, our aim is to show the existence of $t > 0$ such that $\rho[t \vv]_1 \geq 1-\eta$.
This is done by choosing $t = \nicefrac{1}{\zeta(\alpha-1)}$.

Note that $(\alpha-1)t(\evv_1 - \evv_2) \geq 1$ due to our choice of $t$. Then, we will show that with such a $t$, $\rho[t\vv]_1 = 1$ must hold.
For the sake of contradiction, suppose not: $\rho[t\vv]_1 < 1$. Then, by monotonicity of $\rho[t\vv]_j$, we have $\rho[t\vv]_2 > 0$. This means
\begin{equation*}
    \rho[t\vv]_2 = \big[ (\alpha-1)t\evv_2 - \tau(t\vv)  \big]^{\frac{1}{\alpha-1}} > 0,
\end{equation*}
in particular, we have $(\alpha - 1)t\evv_2 - \tau(t\vv) > 0$.
However, recall that $(\alpha-1)t(\evv_1 - \evv_2) \geq 1$, which implies $(\alpha-1)t\evv_1 - \tau(t\vv) > 1$. This results in 
\begin{equation*}
    \rho[t\vv]_1 = \big[ (\alpha-1)t\evv_1 - \tau(t\vv)  \big]^{\frac{1}{\alpha-1}} > 1,
\end{equation*}
thus contradicting $\rho[t\vv]_1 < 1$. Therefore, $\rho[t\vv]_1 = 1$ must hold.

\section{Details of the Step~1 in the proof sketch (\S~\ref{sec:main-proof-sketch})}
\label{sec:proof-lem-main-step1}
We start by formally defining the function class $\overline {\mc F}(\delta)$.
\begin{equation*}
    \overline{\mc F}(\delta) \defeq \left \{ \mZ \mapsto \sum_{\mG \in \sG_\delta} \mA_\mG \indic{\mZ \in \mG+[0,\delta)^{d\times n}}  \mid \mZ \in \sD, \mA_\mG \in \R^{d\times n}\right \},
\end{equation*}
where $\sG_\delta \defeq \{ 0, \delta, \dots, 1-\delta \}^{d \times n}$.
We now state and prove the lemma.
\begin{lemma}
\label{lem:main-step1}
For any $f \in \mc F$ and $\epsilon>0$, there exists a small enough $\delta > 0$ such that there exists $\overline f \in \overline{\mc F}(\delta)$ such that $\funcdist_p(f,\overline f) \leq \epsilon/2$.
\end{lemma}
\begin{proof}
Since $f: \sD \to \reals^{d \times n}$ is a continuous function on a compact domain, it is uniformly continuous. Also, continuity is defined with respect to entry-wise $\ell_p$ norm which is equivalent to entry-wise $\ell_\infty$ norm, uniform continuity leads to
\begin{equation*}
    \forall \epsilon > 0, \exists \delta > 0 \text { such that } \forall \mX, \mY, \linf{\mX-\mY} < \delta \implies \norm{f(\mX)-f(\mY)}_p < \epsilon/2.
\end{equation*}
Then, suppose we create a set of cube grid points $\sG_\delta \defeq \{ 0, \delta, \dots, 1-\delta \}^{d \times n}$, and define a piece-wise constant approximation 
\begin{equation*}
    \overline f(\mX) = \sum\nolimits_{\mG \in \sG_{\delta}} f(\mG) \indic{\mX \in \mG + [0,\delta)^{d \times n}}.
\end{equation*}
Note that for any $\mX \in \mG + [0,\delta)^{d \times n}$ we have $\linf{\mX - \mG} < \delta$, so we have
\begin{equation*}
    \norm{f(\mX) - \overline f(\mX)}_p = \norm{f(\mX)-f(\mG)}_p < \epsilon/2.
\end{equation*}
This implies that
\begin{equation*}
    \funcdist_p(f, \overline f) = 
    \left (\int_{\sD} \norm{f(\mX) - \overline f(\mX)}_p^p \right )^{1/p}
    \leq \epsilon/2,
\end{equation*}
finishing the proof of the lemma.
\end{proof}

\section{Proof of Lemma~\ref{lem:main-step2} (Step~2 in \S~\ref{sec:main-proof-sketch})}
\label{sec:proof-lem-main-step2}
In this section, we describe in further details 
how modified sparse Transformers (the class $\overline{\mc {ST}}^{2,1,1}$) are able to exactly express arbitrary piecewise constant functions in $\overline{\mc F}(\delta)$.
We show that we can compute a \emph{contextual mapping} of the entire input sequences without relying on dense self-attention layers. The token-wise feed-forward layers then transform these contextual mappings to the desired output sequence.

To give a high level summary of the proof, we want to show that 
given a piece-wise constant function $\overline f \in \overline{\mc F}(\delta)$, there exists a modified Transformer network $\overline{g} \in \overline{\mc {ST}}^{2, 1, 1}$ that exactly represents $\overline f$.
Recall first that the function class $\overline{\mc {ST}}^{2, 1, 1}$ has an additive positional embedding matrix $\mE \in \reals^{d \times n}$ that is added to input before the input is fed to the network. We start by choosing the positional embedding $\mE$ and construct a Transformer network that implements quantization of the input, contextual mapping of the quantized input, and value mapping of the context ids.
\begin{enumerate}
\setlength{\itemsep}{2pt}
\item Choose the positional embedding $\mE$ according to $\gamma$ in Assumption~\ref{assm:spsptrn}.\ref{assm:spsptrn-cond2}. After addition, each column of the input $\mX_k+\mE_k$ are in disjoint intervals.
\item Given the input $\mX+\mE$, a series of modified feed-forward layers quantizes it so that each entry of the quantized input has a value in $\{0, \delta, \dots, n-\delta\}$ (Lemma~\ref{lem:quantize}).
\item Next, a series of modified sparse self-attention layers takes the quantized input $\mH$ and implement a {\em contextual mapping} $q$ such that, for different quantized input sequences $\mH$ and $\mH'$, all the elements in $q(\mH)$ and $q(\mH')$ are distinct (Lemma~\ref{lem:contextmap}).
\item Finally, a series of modified feed-forward layers maps each element in the context id $q(\mH)$ to the desired output value of $\overline{f} \in \overline{\mc F}$ at the input $\mX$ (Lemma~\ref{lem:valuemap}).
\end{enumerate}
We defer the proofs of Lemmas~\ref{lem:quantize}, \ref{lem:contextmap}, and \ref{lem:valuemap} to a separate section: see \S~\ref{sec:proof-lemmas}.

Before discussing the details of each step, we note that although a Transformer network stacks self-attention and feed-forward layers in an alternate manner, we can use a series of arbitrary number of the same layers, thanks to skip connections.
The outline of the proof is similar to \citep{yun2019transformers}, but key component in their proof called selective shift operation relies on the fact that each token can attend to the entire sequence; this is not true in sparse Transformers, which poses a nontrivial challenge. We overcome this issue by a more careful construction of the positional embedding $\mE$ and sparse self-attention layers.


\subsection{Choosing the positional embedding}
\label{sec:posembed}
Recall from Assumption~\ref{assm:spsptrn}.\ref{assm:spsptrn-cond2} that there exists a permutation $\gamma: [n] \to [n]$ such that for all $i \in [n-1]$, $\gamma(i)$ is one of the tokens that the $\gamma(i+1)$-th token directly attends to. Using this permutation $\gamma$, we choose the columns of positional embedding $\mE$ in the following way:
\begin{equation*}
    \mE_{\gamma(1)} = (n-1)\ones_n, \text{ and } \mE_{\gamma(i)} = (i-2)\ones_n, \text{ for } i \in [2:n]
\end{equation*}
As a result, the $\gamma(1)$-th column of $\mX + \mE$ will be in the range $[n-1,n)^d$, and similarly $\mX_{\gamma(i)} + \mE_{\gamma(i)} \in [i-2,i-1)^d$ for $i \in [2:n]$. This means that the entries corresponding to different tokens lie be in disjoint intervals of the form $[j,j+1)$, where $j\in[0:n-1]$.

\subsection{Quantization by feed-forward layers}
Note from the previous step that each entry of $\mX + \mE$ must be in $[0,n)$. Next, we quantize this interval $[0,n)$ of input using to a set of $\delta$-grid points $\{0, \delta, \dots, n-\delta\}$. This allows us to deal with finite set of values, which proves useful in the later stages of the proof.
The next lemma shows that the quantization can be carried out using a seried of the modified feed-forward layers.
\begin{lemma}
\label{lem:quantize}
Consider a entry-wise quantization map $g^{\rm ent}_{q} : \reals \to \reals$:
\begin{equation*}
    g_{\rm q}^{\rm ent} (t) = 
    \begin{cases}
    k\delta & \text{ if } k\delta \leq t < (k+1)\delta, ~~k \in [0:n/\delta-1],\\
    t & \text{ otherwise. }
    \end{cases}
\end{equation*}
There exists a function $g_{\rm q}: \reals^{d \times n} \mapsto \reals^{d \times n}$ composed of $\frac{dn}{\delta}$ token-wise feed-forward layers with $r=1$ and an activation $\phi \in \Phi$, which implements the entry-wise quantization $g^{\rm ent}_{q}$ to each entry of its input.
\end{lemma}

\subsection{Contextual mapping by sparse self-attention layers}
After the input $\mX + \mE$ is quantized, the output of $g_{\rm q}$ must be in the following set $\sH_{\delta} \subset \reals^{d \times n}$:
\begin{align*}
{\sH}_{\delta}
\defeq 
\{ \mG + \mE \in \reals^{d \times n} \mid 
&~\mG \in \sG_\delta \},
\end{align*}
where $\sG_\delta \defeq \{ 0, \delta, \dots, 1-\delta \}^{d \times n}$ was defined to be the $\delta$-cubic grid points of $[0,1)^{d \times n}$.
Using this finite set of sequences, we construct a \emph{contextual mapping} that maps each sequence in $\sH_{\delta}$ to unique numbers.
Recall that the sparse attention layer has $p$ sparsity patterns that rotate in cycles, and Assumption~\ref{assm:spsptrn}.\ref{assm:spsptrn-cond3} assumes that one token directly/indirectly access all the other tokens after $s$ such sparse attention layers. We now state the lemma.
\begin{lemma}
\label{lem:contextmap}
Assume that $n \geq 2$, and $\delta^{-1}$ is an integer satisfying $\delta^{-1} \geq 2$. Suppose that the sparse self-attention layers \textup{($h = 2, m = 1$)} satisfy Assumption~\ref{assm:spsptrn} and employ the hardmax $\hdmx$ operator, and that the positional embedding $\mE$ was chosen as described in \S~\ref{sec:posembed}. 
Then, there exist a function $g_{\rm c}: \reals^{d \times n} \to \reals^{d \times n}$ composed of $\frac{p(n-1)}{\delta^{d}} + s$ sparse self-attention layers, and a vector $\vu \in \reals^d$, such that $q(\mH) \defeq \vu^T g_{\rm c}(\mH)$ satisfies the following properties:
\begin{enumerate}[\hspace{5pt}1.]
\setlength{\itemsep}{-3pt}
    \item \label{lemcond:1} For any $\mH \in {\sH}_{\delta}$, the entries of $q(\mH)$ are all distinct.
    \item \label{lemcond:2} For any $\mH, \mH' \in {\sH}_{\delta}$ such that $\mH \neq \mH'$, all entries of $q(\mH)$, $q(\mH')$ are distinct.
\end{enumerate}
\end{lemma}
This contextual mapping maps each unique sequence/context into different context ids, enabling the network to distinguish the same token appearing in different sequences.

\subsection{Value mapping by feed-forward layers}
After the contextual mapping, we use the token-wise feed-forward layers to map each different context ids to the desired output value of the target function $\overline{f}$. More specifically, recall the function $g_{\rm c}$ from Lemma~\ref{lem:contextmap}. For any $\mH \in \sH_\delta$, we need to map the output $g_{\rm c}(\mH)$ of Lemma~\ref{lem:contextmap} to the desired function value $\overline f(\mH-\mE)$ (recall that $\mH$ is the quantized input \emph{after} adding $\mE$ to $\mX$, so we need to subtract $\mE$). This is done by implementing a token-wise value mapping using the feed-forward layers.
\begin{lemma}
\label{lem:valuemap}
There exists a function $g_{\rm v}: \reals^{d \times n} \to \reals^{d \times n}$ composed of $n (\frac{1}{\delta})^{dn}$ token-wise feed-forward layers \textup{($r=1$)} with an activation $\phi' \in \Phi$ such that $g_{\rm v}$ is defined by a token-wise function $g_{\rm v}^{\rm tkn}:\reals^d \to \reals^d$ on each column,
\begin{equation*}
    g_{\rm v} (\mZ) = 
    \begin{bmatrix}
    g_{\rm v}^{\rm tkn} (\mZ_{1}) &
    \cdots &
    g_{\rm v}^{\rm tkn} (\mZ_{n})
    \end{bmatrix},
\end{equation*}
where for all $\mH \in \sH_{\delta}$ and $k \in \{1, \dots, n\}$,
\begin{equation*}
    g_{\rm v}^{\rm tkn}(g_{\rm c}(\mH)_{k}) =
    \overline f (\mH-\mE)_{k}.
\end{equation*}
\end{lemma}

\subsection{Finishing the proof}
Given Lemmas~\ref{lem:quantize}, \ref{lem:contextmap}, and \ref{lem:valuemap}, one can easily check that for any $\mG \in \sG_\delta \defeq \{0, \delta, \dots, 1-\delta\}^{d \times n}$ and any input value $\mX \in \mG + [0,\delta)^{d \times n}$, we have
\begin{align*}
    g_{\rm v}\circ g_{\rm c} \circ g_{\rm q}(\mX + \mE) 
    &= g_{\rm v}\circ g_{\rm c}(\mG + \mE) \\
    &= \begin{bmatrix}
    g_{\rm v}^{\rm tkn} (g_{\rm c}(\mG + \mE)_{1}) &
    g_{\rm v}^{\rm tkn} (g_{\rm c}(\mG + \mE)_{2}) &
    \cdots &
    g_{\rm v}^{\rm tkn} (g_{\rm c}(\mG + \mE)_{n})
    \end{bmatrix}\\
    &= \begin{bmatrix}
    \overline f(\mG)_1 &
    \overline f(\mG)_2 &
    \cdots &
    \overline f(\mG)_n
    \end{bmatrix}
    = \overline f(\mG) = \overline f(\mX).
\end{align*}
Therefore, we have constructed a modified sparse Transformer network $\overline{g}(\mX) \defeq g_{\rm v} \circ g_{\rm c} \circ g_{\rm q}(\mX + \mE)$ that satisfies $\overline{g}(\mX) = \overline{f}(\mX)$ for all $\mX \in \sD$, hence proving Lemma~\ref{lem:main-step2}.

\section{Proof of Lemmas~\ref{lem:quantize}, \ref{lem:contextmap}, and \ref{lem:valuemap}}
\label{sec:proof-lemmas}
\subsection{Proof of Lemma~\ref{lem:quantize}}
\label{sec:proof-lem-quantize}
The proof goes as follows. Using $\frac{n}{\delta}$ token-wise feed-forward layers, we implement the quantization function $g_{\rm q}^{\rm ent}$ that quantizes the first row of the input. Then we stack another $\frac{n}{\delta}$ layers to quantize the second row, and so on.

For the first row, we add $n/\delta$ layers of the following form, for $k \in [0:n/\delta-1]$.
\begin{equation*}
    \mZ \mapsto \mZ + \ve^{(1)} \phi( (\ve^{(1)})^T \mZ - k\delta \vone_n^T),
    ~~
    \phi(t) = 
    \begin{cases}
    0 & t < 0 \text{ or } t \geq \delta,\\
    -t & 0 \leq t < \delta,
    \end{cases}
\end{equation*}
where $\ve^{(1)} \in \reals^d$ is the first canonical basis vector $\ve^{(1)} = (1,0,\dots,0)$.
Each layer quantizes $\mZ_{1,:}$ in $[k\delta, k\delta+\delta)$ to $k\delta$, without modifying other intervals or other rows of $\mZ$.
Note that the activation $\phi$ is a piecewise linear function with three pieces; hence, $\phi \in \Phi$. Therefore, the layers satisfy the definition of modified feed-forward layers. We can now repeat the same construction for the $d-1$ remaining rows.

\subsection{Proof of Lemma~\ref{lem:contextmap}}
\label{sec:proof-lem-contextmap}
In order to construct a network $g_{\rm c}$ that implements the contextual mapping, we first introduce two operations referred to as the \emph{sparse selective shift operation} and \emph{all-max-shift operation}, implemented by at most two (modified) sparse attention heads of head size $1$. Then, we proceed to stack layers implementing the selective shift operations and all-max-shift operations, and prove that these layers map input $\mH \in \sH_\delta$ to unique context ids.

\subsubsection{Preliminaries}
\paragraph{Sparse selective shift operation.}
Given any vector $\vu \in \R^d$, first consider the following function implementable with a sparse attention head with head size 1 and sparsity pattern $\{\mc A_k^l\}_{k \in [n]}$. For $k\in [n]$, the function $\psi^l : \R^{d \times n} \to \R^{1 \times n}$ computes each of its output column in the following way:
\begin{equation*}
    \psi^l(\mZ; b_Q)_k 
    \defeq \vu^T \mZ_{\mc A_k^l} \hdmx [(\vu^T \mZ_{\mc A_k^l})^T (\vu^T \mZ_k - b_Q)]
    =
    \begin{cases}
    \max_{j \in \mc A_k^l} \vu^T \mZ_{j} & \text{ if } \vu^T \mZ_k > b_Q,\\
    \min_{j \in \mc A_k^l} \vu^T \mZ_{j} & \text{ if } \vu^T \mZ_k < b_Q.
    \end{cases}
\end{equation*}
One can consider a sparse self-attention layer that consists of two such heads, with $b_Q < b'_Q$:
\begin{equation*}
    \Psi^l(\mZ; c, b_Q, b'_Q) \defeq
    \mZ + 
    \begin{bmatrix}
    c\ve^{(1)} & -c \ve^{(1)}
    \end{bmatrix}
    \begin{bmatrix}
    \psi^l(\mZ; b_Q)\\ \psi^l(\mZ; b'_Q)
    \end{bmatrix}.
\end{equation*}
The $(1,k)$-th entry of $\Psi^l(\mZ; c, b_Q, b'_Q)$ reads
\begin{align*}
    \Psi^l(\mZ; c, b_Q, b'_Q)_{1,k} &=
    \emZ_{1,k} + c(\psi^l(\mZ; b_Q)_k - \psi^l(\mZ; b'_Q)_k)\\
    &=
    \begin{cases}
    \emZ_{1,k} + c(\max_{j \in \mc A_k^l} \vu^T \mZ_{j} - \min_{j \in \mc A_k^l} \vu^T \mZ_{j}) & \text{ if }  b_Q < \vu^T \mZ_k < b'_Q,\\
    \emZ_{1,k} & \text { if } \vu^T \mZ_k \notin [b_Q, b'_Q].
    \end{cases}
\end{align*}
This means that for input columns $\mZ_k$ satisfying $\vu^T \mZ_k \in (b_Q, b'_Q)$ \emph{only}, $\Psi^l$ shifts up the first entry of $\mZ_k$ by the difference of maximum and minimum values of $\vu^T \mZ_j$ over the sparsity pattern $j \in \mc A_k^l$, \textbf{while leaving other columns intact}. By choosing $b_Q$ and $b'_Q$ properly, we can selectively modify certain columns without touching other columns; we refer to this operation $\Psi^l$ as the \emph{sparse selective shift operation}, and we will see later that this is indeed the key ingredient of our proof.

In fact, this operation is a sparse version of the selective shift operation used in \citep{yun2019transformers}. Since $\mc A_k^l$ is usually only a small subset of $[n]$, one cannot calculate the maximum and minimum of $\vu^T \mZ_j$ over the whole sequence, as done in \citep{yun2019transformers}. Instead, we use Assumption~\ref{assm:spsptrn}.\ref{assm:spsptrn-cond2} and a more careful choice of $\mE$ to get around the restriction posed by sparsity.

\paragraph{All-max-shift operation.}
Suppose the input $\mZ\in \R^{d \times n}$ satisfies $\vu^T \mZ > 0$ entry-wise, for a vector $\vu \in \R^d$. 
Then, the \emph{all-max-shift operation} $\Omega^l:\R^{d \times n} \to \R^{d \times n}$ is a sparse self-attention layer that consists of one attention head:
\begin{equation*}
    \Omega^l (\mZ;c) = \mZ + c \ve^{(1)} \psi^l(\mZ;0).
\end{equation*}
The $(1,k)$-th entry of $\Omega^l(\mZ;c)$ reads
\begin{align*}
    \Omega^l (\mZ;c)_{1,k} &=
    \emZ_{1,k} + c\psi^l(\mZ; 0)_k
    =
    \emZ_{1,k} + c\max_{j \in \mc A_k^l} \vu^T \mZ_{j}.
\end{align*}
So, for each column $k$, the all-max-shift operation shifts up the first entry of $\mZ_k$ by the maximum value of $\vu^T \mZ_{j}$ over the sparsity pattern $j \in \mc A_k^l$. Unlike the selective shift operation, the all-max-shift operation is applied to all the columns.

\paragraph{Column ids.}
Recall that the any input to this step is in
\begin{align*}
{\sH}_{\delta}
\defeq 
\{ \mG + \mE \in \reals^{d \times n} \mid 
&~\mG \in \sG_\delta \defeq [0:\delta:1-\delta]^{d \times n}\}.
\end{align*}
Because of the way $\mE$ is chosen according to the permutation $\gamma$ in Assumption~\ref{assm:spsptrn}.\ref{assm:spsptrn-cond2},
for any $\mH \in \sH_\delta$ we have
\begin{align*}
\mH_{\gamma(1)} &\in [n-1:\delta:n-\delta]^d,\\
\mH_{\gamma(i)} &\in [i-2:\delta:i-1-\delta]^d
\text{ for all } i \in [2:n].
\end{align*}
Now consider $\vu \defeq (1, \delta^{-1}, \delta^{-2}, \dots, \delta^{-d+1})$. It is easy to check that
for any $\mH \in \sH_\delta$, the map $\mH_k \mapsto \vu^T \mH_k$ is one-to-one, and
\begin{equation}
\label{eq:range-z}
\begin{aligned}
    \vu^T \mH_{\gamma(1)} &\in \left [(n-1) \sum_{i=0}^{d-1} \delta^{-i} : \delta : (n-1) \sum_{i=0}^{d-1} \delta^{-i} + \delta^{-d+1} - \delta \right ],\\
    \vu^T \mH_{\gamma(i)} & \in \left [(i-2) \sum_{i=0}^{d-1} \delta^{-i} : \delta : (i-2) \sum_{i=0}^{d-1} \delta^{-i} + \delta^{-d+1} - \delta \right ], \text{ for } i \in [2:n].
\end{aligned}
\end{equation}
Hence, for each column $\mH_k$, the inner product $\vu^T \mH_k$ is in an interval disjoint from the other columns. Thus, $\vu^T \mH_k$ can be thought as a ``column id'' that identifies the column's original input value $\mG_k$ as well as its position $k$. Note furthermore that for any $\mH \in \sH_\delta$,
\begin{equation}
\label{eq:order-z}
    \vu^T \mH_{\gamma(2)} < \vu^T \mH_{\gamma(3)} < \dots < \vu^T \mH_{\gamma(n)} < \vu^T \mH_{\gamma(1)}.
\end{equation}

\subsubsection{Construction of layers}
Given these preliminaries, we now describe our construction of $g_{\rm c}$. Recall from Assumption~\ref{assm:spsptrn}.\ref{assm:spsptrn-cond2} that the permutation $\gamma$ satisfies $\gamma(i-1) \in \bigcup_{l=1}^p \mc A^l_{\gamma(i)}$ for $i \in [2:n]$.
From this, for $i \in [2:n]$ we let $l_i \in [p]$ be any index such that $\gamma(i-1) \in \mc A_{\gamma(i)}^{l_i}$. 
For simplicity of notation, let $z_k \defeq \vu^T \mH_k$ for $k \in [n]$ and $\Delta = \sum_{i=0}^{d-1} \delta^{-i}$.

Next, starting from $i = 2$, we want to sequentially stack $\delta^{-d}$ sparse selective shift operations 
\begin{equation*}
\Psi^{l_i}(\cdot;\delta^{-d},b-\delta/2,b+\delta/2),    
\end{equation*}
in increasing order of $b \in \left [(i-2) \Delta : \delta : (i-2) \Delta + \delta^{-d+1} - \delta \right ]$.
That is, we want to add sparse attention layers with sparsity patterns $\mc A_{\gamma(i)}^{l_i}$ that apply the selective shift operation to each possible value of $z_{\gamma(i)}$.
Recall that the sparsity patterns have to cycle from $\mc A_k^1$ to $\mc A_k^p$, so we have to place other remaining $p-1$ sparsity patterns (whose indices are not $l_i$) in between the $\Psi^{l_i}$ layers. This can be done by setting all the other sparse attention layers to be the identity. This way, we stack a total of $p\delta^{-d}$ sparse attention layers for $i=2$, another $p\delta^{-d}$ for $i=3$, and so on, up to $i=n$.

After these layers, we further stack $s$ all-max-shift operations. For $i = 1, \dots, s$, we add all-max-shift operations of the form
\begin{equation*}
    \Omega^{(i-1) \text{ mod } p + 1}(\cdot;2sn\delta^{-nd-1}).
\end{equation*}
Here, the superscript $(i-1) \text{ mod } p + 1$ is there to make sure that we cycle through the sparsity patterns from $1$ to $p$, until we stack $s$ layers in total.
This finishes the construction of our function $g_{\rm c}$ composed of $\frac{p(n-1)}{\delta^d} + s$ sparse self-attention layers.

\subsubsection{Selective shift operations}
We now explain how these stacked self-attention layers implement a contextual mapping.
This subsection will consider the selective shift operations part; all-max-shift operations are described in the next subsection. Suppose that after the input $\mH \in \sH_\delta$ is processed through the first $\frac{p(n-1)}{\delta^d}$ layers, we get $\wt \mH \in \R^{d \times n}$ at the output. We will show at the end of this subsection that the map $\mH \mapsto \vu^T \wt \mH_{\gamma(n)}$ is a one-to-one map for column $\gamma(n)$, so the selective shift operations compute a ``unique id'' for each possible input sequence $\mH \in \sH_\delta$.

\paragraph{First selective shift.}
First consider the first $p\delta^{-d}$ layers. Omitting layers that are identity, they are essentially selective shift operations $\Psi^{l_2}(\cdot;\delta^{-d},b-\delta/2,b+\delta/2)$ for $b \in [0:\delta:\delta^{-d+1}-\delta]$.
Since $[0:\delta:\delta^{-d+1}-\delta]$ is the set of possible values of $z_{\gamma(2)}$, these layers perform selective shift operation on the $\gamma(2)$-th column without changing the other columns.
Each possible value of $\mH_{\gamma(2)}$ undergoes one and only shift operation (by the corresponding layer with $b = \vu^T \mH_{\gamma(2)}$), by which the $(1,\gamma(2))$-th entry of the input is updated.

Recall by Assumption~\ref{assm:spsptrn}.\ref{assm:spsptrn-cond2} that $\gamma(1) \in \mc A_{\gamma(2)}^{l_2}$, and that $z_{\gamma(1)}$ and $z_{\gamma(2)}$ are the maximum and minimum over the whole sequence $z_{1}, \dots, z_{n}$ (see \eqref{eq:order-z}). 
By Assumption~\ref{assm:spsptrn}.\ref{assm:spsptrn-cond1} we also have $\gamma(2) \in \mc A_{\gamma(2)}^{l_2}$.
Since both $\gamma(1)$ and $\gamma(2)$ are in $\mc A_{\gamma(2)}^{l_2}$, the maximum and minimum value of $z_j \defeq \vu^T \mH_j$'s over $j \in \mc A_{\gamma(2)}^{l_2}$ are $z_{\gamma(1)}$ and $z_{\gamma(2)}$, respectively. Therefore, the $(1,\gamma(2))$-th entry of the input matrix is shifted up as follows:
\begin{equation*}
    \wt \emH_{1,\gamma(2)} \defeq \emH_{1,\gamma(2)} + \delta^{-d} ( z_{\gamma(1)} - z_{\gamma(2)}).
\end{equation*}
Let $\wt \mH_{\gamma(2)}$ be the $\gamma(2)$-th column after the shift operation has shifted $\emH_{1,\gamma(2)}$ to $\wt \emH_{1,\gamma(2)}$. Then, define
\begin{equation*}
    \wt z_{\gamma(2)} \defeq \vu^T\wt \mH_{\gamma(2)} =  z_{\gamma(2)} + \delta^{-d} ( z_{\gamma(1)} - z_{\gamma(2)}).
\end{equation*}
Note that $\wt z_{\gamma(2)} > z_{\gamma(1)}$ because
\begin{align*}
    z_{\gamma(2)} + \delta^{-d} ( z_{\gamma(1)} - z_{\gamma(2)}) > z_{\gamma(1)}
    \iff
    (\delta^{-d} - 1) (z_{\gamma(1)} - z_{\gamma(2)}) > 0,
\end{align*}
which is true.
Therefore, $\wt z_{\gamma(2)}$ becomes the new maximum among the current values $z_{\gamma(1)}, \wt z_{\gamma(2)}, z_{\gamma(3)}, \dots, z_{\gamma(n)}$, and the new minimum element is $z_{\gamma(3)}$.

\paragraph{Second selective shift.}
We now consider the next $p\delta^{-d}$ layers, which are essentially $\Psi^{l_3}(\cdot;\delta^{-d},b-\delta/2,b+\delta/2)$ for $b \in [\Delta:\delta:\Delta+\delta^{-d+1}-\delta]$. They apply the shift operation to the $\gamma(3)$-th column. Since we have $\gamma(2), \gamma(3) \in \mc A_{\gamma(3)}^{l_3}$, the shift operation similarly yields
\begin{equation*}
    \wt z_{\gamma(3)} 
    \defeq 
    z_{\gamma(3)} + \delta^{-d} ( \wt z_{\gamma(2)} - z_{\gamma(3)})
    = z_{\gamma(3)} + \delta^{-d} (z_{\gamma(2)}-z_{\gamma(3)}) + \delta^{-2d} (z_{\gamma(1)}-z_{\gamma(2)}).
\end{equation*}
We can also show $\wt z_{\gamma(3)} > \wt z_{\gamma(2)}$, because
\begin{align*}
    z_{\gamma(3)} + \delta^{-d} ( \wt z_{\gamma(2)} - z_{\gamma(3)}) >
    \wt z_{\gamma(2)}
    \iff
    (\delta^{-d}-1) ( \wt z_{\gamma(2)} - z_{\gamma(3)}) > 0.
\end{align*}
So after this operation $\wt z_{\gamma(3)}$ and $z_{\gamma(4)}$ are the new maximum and minimum over the updated sequence $z_{\gamma(1)},\wt z_{\gamma(2)},\wt z_{\gamma(3)}, z_{\gamma(4)}, \dots, z_{\gamma(n)}$. 

\paragraph{Repeating the process.}
The same process continues. The next $p\delta^{-d}$ layers shifts the $\gamma(4)$-th columns and results in $\wt z_{\gamma(4)}$ which is greater than $\wt z_{\gamma(3)}$.
After the first $p(n-1)\delta^{-d}$ layers, all columns except $\gamma(1)$-th column have been shifted, resulting in $z_{\gamma(1)}, \wt z_{\gamma(2)}, \dots, \wt z_{\gamma(n)}$ satisfying
\begin{equation}
\label{eq:wtz-distinct}
    (n-1)\Delta \leq z_{\gamma(1)} < \wt z_{\gamma(2)} < \dots < \wt z_{\gamma(n)}.
\end{equation}
Let us denote the output of the $p(n-1)\delta^{-d}$-th layer as $\wt \mH$.

\paragraph{Selective shifts implement a one-to-one map.}
Next, we prove that the map from $\mH \in \sH_\delta$ to
\begin{equation*}
    \wt z_{\gamma(n)} \defeq \vu^T \wt \mH_{\gamma(n)} = z_{\gamma(n)} + \sum_{i=1}^{n-1} \delta^{-id} (z_{\gamma(n-i)} - z_{\gamma(n+1-i)} )
\end{equation*}
is one-to-one. Recall that for each column $\mH_k$, the map $\mH_k \mapsto \vu^T \mH_k \eqdef z_k$ is one-to-one.
Also, permutation of columns is one-to-one, which implies that it suffices to show that the map $\begin{bmatrix} z_{\gamma(1)} & \dots & z_{\gamma(n)} \end{bmatrix} \mapsto \wt z_{\gamma(n)}$ is one-to-one.

Suppose we have two sequences $\begin{bmatrix} z_{\gamma(1)} & \dots & z_{\gamma(n)} \end{bmatrix}$ and $\begin{bmatrix} z'_{\gamma(1)} & \dots & z'_{\gamma(n)} \end{bmatrix}$ that map to the same value of $\wt z_{\gamma(n)} = \wt z'_{\gamma(n)}$.
Then,
\begin{align*}
    0 = \wt z_{\gamma(n)} - \wt z'_{\gamma(n)} 
    = z_{\gamma(n)} - z'_{\gamma(n)} + \sum_{i=1}^{n-1} \delta^{-id} (z_{\gamma(n-i)} - z_{\gamma(n+1-i)} - z'_{\gamma(n-i)} + z'_{\gamma(n+1-i)}).
\end{align*}

Suppose $z_{\gamma(n)} \neq z'_{\gamma(n)}$. Since they both lie inside $[(n-2)\Delta : \delta : (n-2)\Delta + \delta^{-d+1} - \delta]$, we have
\begin{equation*}
    -\delta^{-d+1}+\delta 
    \leq z_{\gamma(n)} - z'_{\gamma(n)} 
    \leq \delta^{-d+1}-\delta.
\end{equation*}
Note that all the terms other than $z_{\gamma(n)} - z'_{\gamma(n)}$ are of ``coarser resolution.'' For example, the first term 
\begin{equation*}
 \delta^{-d} (z_{\gamma(n-1)} - z_{\gamma(n)} - z'_{\gamma(n-1)} + z'_{\gamma(n)})   
\end{equation*}
in the summation can only take values $0, \delta^{-d+1}, -\delta^{-d+1}, 2\delta^{-d+1}, -2\delta^{-d+1},\dots$, so it can never cancel the difference $z_{\gamma(n)} - z'_{\gamma(n)}$ and make the sum $\wt z_{\gamma(n)} - \wt z'_{\gamma(n)}$ zero. This implies that $z_{\gamma(n)} = z'_{\gamma(n)}$ must hold.

Next, suppose $z_{\gamma(n-1)} \neq z'_{\gamma(n-1)}$. Since we have $z_{\gamma(n)} = z'_{\gamma(n)}$,
\begin{equation*}
-\delta^{-2d+1}
<
\delta^{-d} (z_{\gamma(n-1)} - z_{\gamma(n)} - z'_{\gamma(n-1)} + z'_{\gamma(n)})
= 
\delta^{-d} (z_{\gamma(n-1)} - z'_{\gamma(n-1)})
<
\delta^{-2d+1}.
\end{equation*}
But similarly, any other terms in the summation have coarser resolution than $\delta^{-2d+1}$, so they cannot cancel the difference $\delta^{-d} (z_{\gamma(n-1)} - z'_{\gamma(n-1)})$. Thus $z_{\gamma(n-1)} = z'_{\gamma(n-1)}$ must hold.
Repeating the same argument up to $\gamma(1)$ proves that the two sequences must be equal:  $\begin{bmatrix} z_{\gamma(1)} & \dots & z_{\gamma(n)} \end{bmatrix} = \begin{bmatrix} z'_{\gamma(1)} & \dots & z'_{\gamma(n)} \end{bmatrix}$. This proves that the map $\mH \mapsto \wt z_{\gamma(n)}$ is one-to-one and $\wt z_{\gamma(n)}$ can be seen as the unique id for the input sequence $\mH \in \sH_\delta$.

\subsubsection{All-max-shift operations}
Next, we explain the operation of the $s$ all-max-shift layers. Recall from Assumption~\ref{assm:spsptrn}.\ref{assm:spsptrn-cond3} that any token can attend to all the other tokens after $s$ steps, either directly or indirectly. 
Also recall from the last subsection that the input to the first all-max-shift layer is $\wt \mH$, and the maximum entry of $\vu^T \wt \mH$ is $\wt z_{\gamma(n)}$, the unique id for input $\mH$.
From the statement of Lemma~\ref{lem:contextmap}, the output after the $s$ all-max-shift operations for input $\mH$ is denoted as $g_{\rm c} (\mH)$.
In this subsection, we show that through $s$ all-max-shift operations, the maximum $\wt z_{\gamma(n)}$ will propagate to all tokens and be a ``dominant'' term, which determines the interval that $\vu^T g_{\rm c}(\mH)$ lies in.
As a result, we can show Properties~\ref{lem:contextmap}.\ref{lemcond:1} and \ref{lem:contextmap}.\ref{lemcond:2} of $g_{\rm c}$ at the end.

\paragraph{Some preliminaries.}
Note that the unique id $\wt z_{\gamma(n)}$ has the following upper bound:
\begin{align}
    \wt z_{\gamma(n)} 
    &\defeq 
    z_{\gamma(n)} + \sum_{i=1}^{n-2} \delta^{-id} (z_{\gamma(n-i)} - z_{\gamma(n+1-i)} )
    + \delta^{-(n-1)d} (z_{\gamma(1)} - z_{\gamma(2)})\nonumber\\
    &\leq z_{\gamma(n)} + \delta^{-d} \sum_{i=1}^{n-2} (z_{\gamma(n-i)} - z_{\gamma(n+1-i)} )
    + \delta^{-(n-1)d} (z_{\gamma(1)} - z_{\gamma(2)})\nonumber\\
    &= z_{\gamma(n)} + \delta^{-d} (z_{\gamma(2)} - z_{\gamma(n)} )
    + \delta^{-(n-1)d} (z_{\gamma(1)} - z_{\gamma(2)})\nonumber\\
    &= \delta^{-(n-1)d} z_{\gamma(1)} - (\delta^{-(n-1)d} - \delta^{-d}) z_{\gamma(2)} - (\delta^{-d} - 1) z_{\gamma(n)}\nonumber\\
    &\leq \delta^{-(n-1)d} z_{\gamma(1)}
    \leq \delta^{-(n-1)d} ((n-1)\Delta + \delta^{-d+1} - \delta)\nonumber\\
    &\leq \delta^{-(n-1)d} ( n-1+\delta ) (\delta^{-d} - 1) \leq \delta^{-nd} - \delta \label{eq:wtzn-bound}
\end{align}
where we used $\Delta \defeq \sum_{i=0}^{d-1} \delta^{-i} = \frac{\delta^{-d}-1}{\delta^{-1}-1} \leq \delta^{-d}-1$.
A similar bound
\begin{equation}
\label{eq:wtzi-bound}
    \wt z_{\gamma(i)} \leq n\delta^{-id} - \delta
\end{equation}
also holds from a similar derivation.
Next, recall from Assumption~\ref{assm:spsptrn}.\ref{assm:spsptrn-cond3} the definitions
\begin{equation*}
    \mc S^1_k \defeq \mc A^1_k,~~
    \mc S^t_k \defeq \bigcup\nolimits_{j \in \mc A^{(t-1) \text{ mod } p + 1}_k} \mc S^{t-1}_j,
\end{equation*}
and that there exists $s \geq 1$ such that, for all $k \in [n]$, $\mc S^{s}_k = [n]$.
Finally, the following inequality will be useful throughout: for any integer $s \geq 1$,
\begin{equation}
\label{eq:s-ineq}
    \left (\frac{2s+1}{2s}\right )
    \leq 
    \left (\frac{2s+1}{2s}\right )^2
    \leq
    \dots \leq 
    \left (\frac{2s+1}{2s}\right )^s \leq 2.
\end{equation}
Let us now describe the operation that the all-max-shift layers $\Omega^{(i-1) \text{ mod } p + 1}(\cdot;2sn\delta^{-nd-1})$, $i = 1, \dots, s$, carry out. 

\paragraph{First all-max-shift.}
The input to the first all-max-shift layer is $\wt \mH$. Let the output of the layer be $\mM^1$. Recall that $\vu^T \wt \mH$ consists of values $z_{\gamma(1)}, \wt z_{\gamma(2)}, \dots, \wt z_{\gamma(n)}$, which are all strictly greater than 0 and strictly less than $n \delta^{-nd}$ (by \eqref{eq:wtzn-bound}).
So, for each column $k \in [n]$, the layer update reads
\begin{equation*}
    \emM^1_{1,k} 
    \defeq \wt \emH_{1,k} + 2sn\delta^{-nd-1} \max_{j \in \mc A_k^1} \vu^T \wt \mH_j
    = \wt \emH_{1,k} + 2sn\delta^{-nd-1} \vu^T \wt \mH_{j_k^1},
\end{equation*}
where $j_k^1 \defeq \argmax_{j \in \mc A_k^1} \vu^T \wt \mH_j$.
After the update, $\vu^T \mM^1_k$ is ``dominated'' by $2sn\delta^{-nd-1} \vu^T \wt \mH_{j_k^1}$, meaning that for any $k, k' \in [n]$, 
\begin{equation*}
    \vu^T \wt \mH_{j_k^1} < \vu^T \wt \mH_{j_{k'}^1}
    \implies
    \vu^T \mM_{k} < \vu^T \mM_{k'}.
\end{equation*}
This is because the minimum gap between different values of $\vu^T \wt \mH_{j_k^1}$ is at least $\delta$, and we have
\begin{equation*}
    \vu^T \wt \mH_k < n\delta^{-nd} < 2sn\delta^{-nd-1} \cdot \delta,
\end{equation*}
so if $\vu^T \wt \mH_{j_k^1} < \vu^T \wt \mH_{j_{k'}^1}$, that solely determines the order $\vu^T \mM_{k} < \vu^T \mM_{k'}$ because $\vu^T \wt \mH_k$ cannot reverse it.
Also, by the definition of $j_k^1$, for any index set $\mc B \in [n]$ we have
\begin{equation}
\label{eq:all-max-shift-1}
    \max_{i \in \mc B} \vu^T \wt \mH_{j_i^1}
    = \max_{j \in \bigcup_{i \in \mc B} \mc A_i^1} \vu^T \wt \mH_j.
\end{equation}
If $s\geq 2$, we move on to the second layer. 

\paragraph{Second all-max-shift.}
At the second all-max-shift, we have sparsity patterns $\mc A_k^{1 \textup{ mod } p + 1}$.
Let us the output of this layer as $\mM^2$.
For each column $k \in [n]$, the layer update reads
\begin{align*}
    \emM^2_{1,k} 
    &\defeq \emM^1_{1,k} + 2sn\delta^{-nd-1} \max_{j \in \mc A_k^{1 \textup{ mod } p + 1}} \vu^T \mM^1_j
    = \emM^1_{1,k} + 2sn\delta^{-nd-1} \vu^T \mM^1_{j_k^2},
\end{align*}
where $j_k^2 \defeq \argmax_{j \in \mc A_k^{1 \textup{ mod } p + 1}} \vu^T \mM^1_j$. If we look at the update more closely, we can apply \eqref{eq:all-max-shift-1} and get
\begin{align*}
    \vu^T \mM^2_{k} &= 
    \vu^T \wt \mH_{k} 
    + 2sn\delta^{-nd-1} \vu^T \wt \mH_{j_k^1}
    + 2sn\delta^{-nd-1} 
    (\vu^T \wt \mH_{j_k^2} + 2sn\delta^{-nd-1} \max_{i \in \mc A_{k}^{1 \textup{ mod } p + 1}} \vu^T \wt \mH_{j_i^1}) \\
    &=    
    \vu^T  \wt \mH_{k} 
    + 2sn\delta^{-nd-1} (\vu^T \wt \mH_{j_k^1} + \vu^T \wt \mH_{j_k^2})
    + (2sn\delta^{-nd-1})^2 \max_{j \in \mc S_{k}^2} \vu^T \wt \mH_j.
\end{align*}
Again, the last term dominates the rest of the terms in $\vu^T \mM^2_k$, because the minimum gap between different values of $\max_{j \in \mc S_{k}^2} \vu^T \wt \mH_j$ is at least $\delta$, and
\begin{align*}
    &~
    \vu^T \mM^2_{k} - (2sn\delta^{-nd-1})^2 \max_{j \in \mc S_{k}^2} \vu^T \wt \mH_j
    =
    \vu^T \wt \mH_{k} 
    + 2sn\delta^{-nd-1} (\vu^T \wt \mH_{j_k^1} + \vu^T \wt \mH_{j_k^2})\\
    <&~
    (1+4sn\delta^{-nd-1}) n\delta^{-nd}
    \leq
    (1+4s) n^2 \delta^{-2nd-1}
    \leq
    (2sn\delta^{-nd-1})^2 \cdot \delta
    = 4s^2 n^2 \delta^{-2nd-1}.
\end{align*}
The last inequality holds due to inequality~\eqref{eq:s-ineq}, because 
\begin{equation*}
    \left ( \frac{2s+1}{2s} \right)^2 \leq 2
    \iff 1+4s \leq 4s^2
\end{equation*}
is true for $s \geq 2$.

\paragraph{Remaining all-max-shifts.}
If $s \geq 3$, we move on to the third layer, which outputs $\mM^3$. Similarly, we can show that $\vu^T \mM^3_k$ is dominated by $(2sn\delta^{-nd-1})^3 \max_{j \in \mc S_{k}^3} \vu^T \wt \mH_j$ because the rest of the terms in $\vu^T \mM^3_k$ is strictly upper-bounded
\begin{equation*}
    \vu^T \mM^3_k - (2sn\delta^{-nd-1})^3 \max_{j \in \mc S_{k}^3} \vu^T \wt \mH_j
    <
    (1+3 \cdot 2sn\delta^{-nd-1}+ 3 \cdot (2sn\delta^{-nd-1})^2 ) n\delta^{-nd-1},
\end{equation*}
which can then be shown to be smaller than $(2sn\delta^{-nd-1})^3 \cdot \delta$:
\begin{equation*}
    (1+3 \cdot 2sn\delta^{-nd-1}+ 3 \cdot (2sn\delta^{-nd-1})^2 ) n\delta^{-nd}
    \leq 
    (1+6s+12s^2) n^3 \delta^{-3nd-2}
    \leq 8s^3 n^3 \delta^{-3nd-3} \cdot \delta.
\end{equation*}
The last inequality is due to the fact that
$1+6s+12s^2 \leq 8s^3$ for $s \geq 3$, which can derived from \eqref{eq:s-ineq}.
Repeating this process, after all $s$ layers we get $\mM^s$, and $\vu^T \mM^s_k$ is dominated by
\begin{equation*}
    (2sn\delta^{-nd-1})^s \max_{j \in \mc S_{k}^s} \vu^T \wt \mH_j
    = 
    (2sn\delta^{-nd-1})^s \max_{j \in [n]} \vu^T \wt \mH_j
    = 
    (2sn\delta^{-nd-1})^s \wt z_{\gamma(n)}.
\end{equation*}
This is because the remaining terms in $\vu^T \mM^s_k$ can be strictly upper-bounded
\begin{equation*}
    \vu^T \mM^s_k - (2sn\delta^{-nd-1})^s \wt z_{\gamma(n)}
    < 
    \left ( \sum_{i=0}^{s-1} \choose{s}{i} (2sn \delta^{-nd-1})^i \right )  n\delta^{-nd},
\end{equation*}
which is then dominated by the smallest difference possible in $(2sn\delta^{-nd-1})^s \wt z_{\gamma(n)}$:
\begin{align*}
    &~
    \left ( \sum_{i=0}^{s-1} \choose{s}{i} (2sn \delta^{-nd-1})^i \right )  n\delta^{-nd}
    \leq\left ( \sum_{i=0}^{s-1} \choose{s}{i} (2s)^i \right ) (n\delta^{-nd-1})^{s-1} n\delta^{-nd}\\
    =&~ ((1+2s)^s - (2s)^s) (n\delta^{-nd-1})^s \cdot \delta \leq (2s n \delta^{-nd-1})^s \cdot \delta.
\end{align*}
The last inequality used $(1+2s)^s - (2s)^s \leq (2s)^s$, derived from \eqref{eq:s-ineq}.

\subsubsection{Verifying Properties~\ref{lem:contextmap}.\ref{lemcond:1} and \ref{lem:contextmap}.\ref{lemcond:2}}
\label{sec:context-map-proof-verify}
After these all-max-shift operations, we define the output $\mM^s$ of the last all-max-shift layers to be the output of the function $g_{\rm c}$ for input $\mH$, i.e., $g_{\rm c} (\mH) \defeq \mM^s$.

Property~\ref{lem:contextmap}.\ref{lemcond:1} requires that for any $\mH \in \sH_\delta$, all the components $\vu^T g_{\rm c} (\mH)$ need to be distinct. This is true, because for each column of $\vu^T g_{\rm c} (\mH)$, we have
\begin{equation*}
    \vu^T g_{\rm c} (\mH)_{k}~\mod~2sn\delta^{-nd}
    = \vu^T \wt \mH_{k}.
\end{equation*}
This is because anything added by the all-max-shift operations is an integer multiple of $2sn\delta^{-nd}$, and $\vu^T \wt \mH_k < n\delta^{-nd} < 2n\delta^{-nd}$ for all $k$.
Recall that $\wt \mH$ is the input matrix for the first max-shift operation, and that the components of $\vu^T \wt \mH$ are $z_{\gamma(1)}, \wt z_{\gamma(2)}, \dots, \wt z_{\gamma(n)}$, which were shown to be distinct by \eqref{eq:wtz-distinct}.
Since $\vu^T g_{\rm c} (\mH)_{k}$ produce distinct outputs for a $\mod$ operation, they themselves have to distinct. This proves Property~\ref{lem:contextmap}.\ref{lemcond:1}.

Also, by the ``domination'' argument in the previous subsection, the output $g_{\rm c} (\mH)$ has the property that for any column, $\vu^T g_{\rm c} (\mH)_k$ lies inside an interval determined by $\wt z_{\gamma(n)}$, the unique id for the input $\mH$:
\begin{equation*}
    \vu^T g_{\rm c} (\mH)_k \in \left [(2sn\delta^{-nd-1})^s \wt z_{\gamma(n)}, (2sn\delta^{-nd-1})^s (\wt z_{\gamma(n)} + \delta ) \right ),
\end{equation*}
and these intervals do not overlap because any different values of $\wt z_{\gamma(n)}$ must differ by at least $\delta$. This means that for any input $\mH, \mH' \in \sH_\delta$, the components in $\vu^T g_{\rm c} (\mH)$ and $\vu^T g_{\rm c} (\mH')$ lie in disjoint intervals. Together with Property~\ref{lem:contextmap}.\ref{lemcond:1}, this proves Property~\ref{lem:contextmap}.\ref{lemcond:2}.

\subsection{Proof of Lemma~\ref{lem:valuemap}}
\label{sec:proof-lem-valuemap}
To prove this lemma, we implement a token-wise function that maps
\begin{equation*}
    g_{\rm v}^{\rm tkn}(g_{\rm c}(\mH)_{k}) =
    \overline f (\mH-\mE)_{k},
\end{equation*}
for all $\mH \in \sH_\delta$ and $k \in [n]$.
From the construction of Lemma~\ref{lem:contextmap}, there are $n|\sH_\delta| = \frac{n}{\delta^{dn}}$ distinct values of $\vu^T g_{\rm c}(\mH)_{k}$, and different values of $\vu^T g_{\rm c}(\mH)_{k}$ differ by at least $\delta$.
The implementation of $g_{\rm v}^{\rm tkn}$ can be done by stacking feed-forward layers so that each layer maps one unique number to the corresponding output column.

More precisely, choose any $\mH \in \sH_\delta$. For each of the $n$ values of $\vu^T g_{\rm c}(\mH)_k$, we add one feed-forward layer of the form
\begin{equation*}
    \mZ \mapsto \mZ + (\overline f (\mH-\mE)_{k} - g_{\rm c}(\mH)_{k}) \phi' (\vu^T \mZ - \vu^T g_{\rm c}(\mH)_{k} \vone_n^T),
    ~~
    \phi'(t) = 
    \begin{cases}
    0 & t < -\delta/2 \text{ or } t \geq \delta/2,\\
    1 & -\delta/2 \leq t < \delta/2.
    \end{cases}
\end{equation*}
This layer updates any column $j$ of its input $\mZ$ that satisfies $\vu^T g_{\rm c}(\mH)_{k}-\delta/2 \leq \vu^T \mZ_j < \vu^T g_{\rm c}(\mH)_{k}+\delta/2$, without modifying any other columns that are out of this range.

We stack these layers for all possible values of $\mH \in \sH_\delta$. After $\frac{n}{\delta^{dn}}$ such layers, we get the desired function $g_{\rm v}$ that satisfies 
\begin{equation*}
    g_{\rm v} (\mZ) = 
    \begin{bmatrix}
    g_{\rm v}^{\rm tkn} (\mZ_{1}) &
    \cdots &
    g_{\rm v}^{\rm tkn} (\mZ_{n})
    \end{bmatrix},
\end{equation*}
where for all $\mH \in \sH_{\delta}$ and $k \in [n]$,
\begin{equation*}
    g_{\rm v}^{\rm tkn}(g_{\rm c}(\mH)_{k}) =
    \overline f (\mH-\mE)_{k}.
\end{equation*}

\section{Proof of Lemma~\ref{lem:main-step3} (Step~3 in \S~\ref{sec:main-proof-sketch})}
\label{sec:proof-lem-main-step3}
In this section, we describe how the modified sparse Transformer network $\overline g \in \overline{\mc{ST}}^{2,1,1}$ constructed in Lemma~\ref{lem:main-step2} can be approximated with an original sparse Transformer network $g \in \mc{ST}^{2,1,4}$.
Recall that $\overline g$ is a ``modified'' sparse Transformer network, which employ the hardmax $\hdmx$ operators in place of $\rho$ operators in sparse  self-attention layers and piecewise linear activations $\phi \in \Phi$ instead of $\relu$s in feed-forward layers. The goal of this lemma is to approximate the function $\overline g = g_{\rm v} \circ g_{\rm c} \circ g_{\rm q} \in \overline{\mc {ST}}^{2,1,1}$ with a standard sparse Transformer $g = \wt g_{\rm v} \circ \wt g_{\rm c} \circ \wt g_{\rm q}\in \mc {ST}^{2,1,4}$ with accuracy  $\funcdist_p(\overline g, g) \leq \epsilon/2$.
As the construction of $\overline g$ consists of three steps, we will approximate each of them step by step.
The whole intuition behind the proof is that as long as we are considering $L_p$ approximation, we can approximate $\hdmx$ and $\phi \in \Phi$ as closely as we want with $\rho$ and $\relu$s, respectively. However, as the proof will show, controlling the aggregated error over layers is not a trivial job.

\subsection{Approximating the quantization function $g_{\rm q}$ (Lemma~\ref{lem:quantize})}
\label{sec:proof-lem-step3-quantize}
We first consider approximating $g_{\rm q}$ from Lemma~\ref{lem:quantize} with a standard feed-forward layer counterpart, $\wt g_{\rm q}$.
Recall from \S~\ref{sec:proof-lem-quantize} that the modified feed-forward layers used in $g_{\rm q}$ are of the form
\begin{equation}
\label{eq:gq-layer}
    \mZ \mapsto \mZ + \ve^{(i)} \phi( (\ve^{(i)})^T \mZ - k\delta \vone_n^T),
    ~~
    \phi(t) = 
    \begin{cases}
    0 & t < 0 \text{ or } t \geq \delta,\\
    -t & 0 \leq t < \delta,
    \end{cases}
\end{equation}
for $i \in [d]$ and $k \in [0:n/\delta-1]$.
Note that the activation $\phi \in \Phi$ can be closely approximated by three $\relu$s:
\begin{align*}
    \wt \phi_\alpha(t) 
    &\defeq
    -\relu(t) + \frac{1}{\alpha} \relu (t - (1-\alpha) \delta) - \frac{1-\alpha}{\alpha} \relu (t - \delta)\\
    &=
    \begin{cases}
    0 & t \leq 0 \text { or } t \geq \delta,\\
    -t & 0 \leq t \leq (1-\alpha) \delta,\\
    \frac{1-\alpha}{\alpha}(t-\delta) & (1-\alpha) \delta \leq t \leq \delta,
    \end{cases}
\end{align*}
where $0 <\alpha< 1$.
Note that $\wt \phi_\alpha(t) = \phi(t)$ except for an interval $((1-\alpha)\delta, \delta)$, and by shrinking $\alpha > 0$ this interval can be made arbitrarily small.
Consider approximating the layers \eqref{eq:gq-layer} with standard feed-forward layers, by replacing $\phi$ with its approximation $\wt \phi_\alpha$. Let the resulting function be $\wt g_{\rm q} \in \mc {ST}^{2,1,3}$.

Then, it is easy to check that $g_{\rm q}(\mX + \mE) = \wt g_{\rm q}(\mX + \mE)$ holds if all coordinates of $\mX \in [0,1)^{d \times n}$ are in the intervals of the form $[k\delta, (k+1-\alpha)\delta]$ for some $k \in [0: n/\delta - 1]$; i.e., the intervals in which $\wt \phi_\alpha$ perfectly approximates $\phi$. The Lebesgue measure of the set of such inputs $\mX$ is
\begin{equation*}
    ((1-\alpha)\delta)^{nd} \times \frac{1}{\delta^{nd}}
    = (1-\alpha)^{nd},
\end{equation*}
and this can be made arbitrarily close to 1 by making $\alpha$ small.
As a result, ``most'' of the input $\mX \in \sD$ satisfies $g_{\rm q}(\mX + \mE) = \wt g_{\rm q}(\mX + \mE) \in \sH_\delta$, while a small fraction (of measure at most $1-(1-\alpha)^{nd}$) can map to some other values. 
For most of the remaining of the proof, we will consider the fraction of inputs mapped correctly to $\sH_\delta$ and bound their approximation error. We will come back to the $1-(1-\alpha)^{nd}$ fraction at the end of the proof.

\subsection{Approximating the contextual mapping $g_{\rm c}$ (Lemma~\ref{lem:contextmap})}
Let us now consider approximating the contextual mapping $g_{\rm c}$ in Lemma~\ref{lem:contextmap}, constructed using the hardmax $\hdmx$ operators, with the standard sparse self-attention layers employing $\rho$ operator. We will call the approximation $\wt g_{\rm c}$.
Recall that $\rho$ satisfies Assumption~\ref{assm:rhotohdmx}:
\rhotohdmx*
This means that $\rho$ can closely approximate $\hdmx$ in the sense that whenever the input vector $\vv$ to the $\rho$ operator has a maximum element $\evv_{j^*}$ by some margin $\zeta$, then the $j^*$-th component of the output $\rho[t \vv]$ is close to $1$, while the other components of $\rho[t \vv]$ are close to $0$.

Recall that $g_{\rm c}$ consists of two parts. The first part is a composition of sparse selective shift operations, and the second is a composition of all-max-shift operations. We will first examine how ``errors'' are introduced when $\hdmx$ is replaced with $\rho$ in both operations, discuss how the errors accumulate, and show how to choose the right $\zeta$ and $\eta$ to control the errors in the approximation $\wt g_{\rm c}$.

\paragraph{Errors introduced by $\rho$: Sparse selective shift operation.}
Recall that the key component in both the selective shift operation and all-max-shift operation is the sparse attention head $\psi^l(\cdot)$, which computes its $k$-th column as the following:
\begin{equation*}
    \psi^l(\mZ; b_Q)_k 
    \defeq \vu^T \mZ_{\mc A_k^l} \hdmx [(\vu^T \mZ_{\mc A_k^l})^T (\vu^T \mZ_k - b_Q)]
    =
    \begin{cases}
    \max_{j \in \mc A_k^l} \vu^T \mZ_{j} & \text{ if } \vu^T \mZ_k > b_Q,\\
    \min_{j \in \mc A_k^l} \vu^T \mZ_{j} & \text{ if } \vu^T \mZ_k < b_Q.
    \end{cases}
\end{equation*}
Now suppose we replaced $\hdmx$ with $\rho$ satisfying Assumption~\ref{assm:rhotohdmx}. 
Suppose each entry in $\vu^T \mZ$ differs at least by $\delta$, which is true in the construction of $g_{\rm c}$. We choose $\zeta = \delta/2$ and some $0<\eta<1$, and corresponding $t > 0$. Then, replace $\hdmx[\cdot]$ with $\rho[t\cdot]$ and define
\begin{equation*}
    \wt \psi^l(\mZ; b_Q)_k
    \defeq \vu^T \mZ_{\mc A_k^l} \rho [t(\vu^T \mZ_{\mc A_k^l})^T (\vu^T \mZ_k - b_Q)].
\end{equation*}
If $\vu^T \mZ_k > b_Q$, it is easy to check that $\wt \psi^l(\mZ; b_Q)_k$ satisfies
\begin{equation}
\label{eq:rho-approx-bd}
    (1-\eta) \max_{j \in \mc A_k^l} \vu^T \mZ_{j} 
    + \eta \min_{j \in \mc A_k^l} \vu^T \mZ_{j}
    \leq \wt \psi^l(\mZ; b_Q)_k
    \leq \max_{j \in \mc A_k^l} \vu^T \mZ_{j}.
\end{equation}
Similarly, if $\vu^T \mZ_k < b_Q$, we have
\begin{equation*}
    \min_{j \in \mc A_k^l} \vu^T \mZ_{j}
    \leq \wt \psi^l(\mZ; b_Q)_k
    \leq (1-\eta) \min_{j \in \mc A_k^l} \vu^T \mZ_{j} 
    + \eta \max_{j \in \mc A_k^l} \vu^T \mZ_{j}.
\end{equation*}
Now consider the \emph{approximate} sparse selective shift operator $\wt \Psi^l$, implemented with $\wt \psi^l$. For $b_Q < b'_Q$, we define
\begin{equation*}
    \wt \Psi^l(\mZ; c, b_Q, b'_Q) \defeq
    \mZ + 
    \begin{bmatrix}
    c\ve^{(1)} & -c \ve^{(1)}
    \end{bmatrix}
    \begin{bmatrix}
    \wt \psi^l(\mZ; b_Q)\\ \wt \psi^l(\mZ; b'_Q)
    \end{bmatrix}.
\end{equation*}
For any column $\mZ_k$ satisfying $b_Q < \vu^T \mZ_k < b'_Q$, we have
\begin{equation*}
    (1-2 \eta) \left (\max_{j \in \mc A_k^l} \vu^T \mZ_{j} 
    - \min_{j \in \mc A_k^l} \vu^T \mZ_{j} \right )
    \leq
    \wt \psi^l(\mZ; b_Q)_k - \wt \psi^l(\mZ; b'_Q)_k
    \leq \max_{j \in \mc A_k^l} \vu^T \mZ_{j} - \min_{j \in \mc A_k^l} \vu^T \mZ_{j},
\end{equation*}
and for any column $\mZ_k$ satisfying $\vu^T \mZ_k \notin [b_Q, b'_Q]$, we get
\begin{equation*}
    |\wt \psi^l(\mZ; b_Q)_k - \wt \psi^l(\mZ; b'_Q)_k|
    \leq \eta \left (\max_{j \in \mc A_k^l} \vu^T \mZ_{j} 
    - \min_{j \in \mc A_k^l} \vu^T \mZ_{j} \right ).
\end{equation*}
Recall that for the hardmax $\hdmx$ version, we had
\begin{align*}
    \psi^l(\mZ; b_Q)_k - \psi^l(\mZ; b'_Q)_k
    =
    \begin{cases}
    \max_{j \in \mc A_k^l} \vu^T \mZ_{j} - \min_{j \in \mc A_k^l} \vu^T \mZ_{j} & \text{ if }  b_Q < \vu^T \mZ_k < b'_Q,\\
    0 & \text { if } \vu^T \mZ_k \notin [b_Q, b'_Q].
    \end{cases}
\end{align*}
From this observation, the approximation error $\wt \Psi^l - \Psi^l$ of the selective shift operator on the $(j,k)$-th entry of the output can be bounded as follows:
\begin{align*}
    \wt \Psi^l(\mZ; c, b_Q, b'_Q)_{j,k} - \Psi^l(\mZ; c, b_Q, b'_Q)_{j,k}
    \in 
    \begin{cases}
        [-2c\eta D_k^l, 0] & \text{ if } j=1, \vu^T \mZ_k \in (b_Q, b'_Q),\\
        [-c\eta D_k^l, c\eta D_k^l] & \text { if } j=1, \vu^T \mZ_k \notin [b_Q, b'_Q],\\
        \{0\} & \text { if } j \neq 1,
    \end{cases}
\end{align*}
where we used $D_k^l \defeq \max_{j \in \mc A_k^l} \vu^T \mZ_{j} - \min_{j \in \mc A_k^l} \vu^T \mZ_{j}$ for simplicity.


\paragraph{Errors introduced by $\rho$: All-max-shift operation.}
Next, we examine the approximation error of the all-max-shift operation introduced by replacement of $\hdmx$ with $\rho$. Let us define the \emph{approximate} all-max-shift operation $\wt \Omega^l$:
\begin{equation*}
    \wt \Omega^l (\mZ;c) = \mZ + c \ve^{(1)} \wt \psi^l(\mZ;0).
\end{equation*}
From \eqref{eq:rho-approx-bd}, we can check that the approximation error $\wt \Omega^l - \Omega^l$ of the all-max-shift operation is bounded as
\begin{align*}
    \wt \Omega^l (\mZ;c)_{j,k} - \Omega^l (\mZ;c)_{j,k}
    \in 
    \begin{cases}
        [-c\eta D_k^l, 0] & \text{ if } j=1,\\
        \{0\} & \text { if } j \neq 1.
    \end{cases}
\end{align*}


\paragraph{Errors in selective shift operations.}
Given these approximation error bounds of single operations, we now analyze the accumulation of errors through multiple layers. We first consider the first $p\delta^{-d}$ self-attention layers in $g_{\rm c}$. 
Recall that they consist of selective shift layers $\Psi^{l_2}(\cdot;\delta^{-d},b-\delta/2,b+\delta/2)$ for $b \in [0:\delta:\delta^{-d+1}-\delta]$ and $(p-1)\delta^{-d}$ identity layers. A natural way to approximate these layers with standard self-attention layers is to use approximate layers $\wt \Psi^{l_2}(\cdot;\delta^{-d},b-\delta/2,b+\delta/2)$, with sufficiently large $t>0$.
As we have seen above, there is no error introduced by $\rho$ except for the first row. Thus, we will analyze the approximation error of $\wt \Psi^{l_2}(\cdot;\delta^{-d},b-\delta/2,b+\delta/2)$ for the first row only. 

Let us remind the readers how the first selective shift operation (done by the first $p\delta^{-d}$ layers) originally worked in $g_{\rm c}$.
The input to $g_{\rm c}$ is $\mH$, and we define $z_k \defeq \vu^T \mH_k$ and $\Delta = \sum_{i=0}^{d-1} \delta^{-i}$. Recall from Eqs.~\eqref{eq:range-z} and \eqref{eq:order-z} in \S~\ref{sec:proof-lem-contextmap} that
\begin{equation*}
    0 \leq z_{\gamma(2)} < z_{\gamma(3)} < \dots < z_{\gamma(n)} < z_{\gamma(1)} \leq (n-1)\Delta + \delta^{-d+1} - \delta < n \delta^{-d} 
\end{equation*}
and $z_{\gamma(2)} \in [0 :\delta:\delta^{-d+1}-\delta]$,
so $z_{\gamma(2)}$ will undergo the selective shift by one of the self-attention layers, which updates the $(1,\gamma(2))$-th entry of the input. Let $\wt \mH_{\gamma(2)}$ be the updated value of the column and $\wt z_{\gamma(2)} \defeq \vu^T \wt \mH_{\gamma(2)}$. The new sequence satisfies
\begin{equation*}
    \Delta \leq z_{\gamma(3)} < \dots < z_{\gamma(n)} < z_{\gamma(1)} < \wt z_{\gamma(2)} < n\delta^{-2d},
\end{equation*}
where the strict upper bound on $\wt z_{\gamma(2)}$ is from Eq.~\eqref{eq:wtzi-bound}.

In case of the approximation $\wt \Psi^{l_2}$, we have seen that the error depends on the gap between maximum and minimum of $\vu^T \mZ_j$'s, and this gap may grow larger as error accumulates; in the worst case, it may grow exponentially. To see this, suppose $a_0$ and $b_0$ are the maximum and minimum value of $\vu^T \mZ_j$'s, and they go through a selective shift operation, but they do not belong to the range of the operation $(b_Q, b'_Q)$. Then, $a_0$ and $b_0$ will be updated to $a_1$ and $b_1$, which are bounded by
\begin{equation*}
    a_1 \leq a_0 + \delta^{-d} \eta (a_0 - b_0),~~
    b_1 \geq b_0 - \delta^{-d} \eta (a_0 - b_0).
\end{equation*}
After the next layer, we get
\begin{align*}
    a_2 &\leq a_1 + \delta^{-d} \eta (a_1 - b_1) \leq a_0 + \delta^{-d} \eta (a_0 - b_0) + \delta^{-d} \eta (1 + 2 \delta^{-d} \eta) (a_0 - b_0),\\
    b_2 &\geq b_1 - \delta^{-d} \eta (a_1 - b_1) \geq b_0 - \delta^{-d} \eta (a_0 - b_0) - \delta^{-d} \eta (1 + 2 \delta^{-d} \eta) (a_0 - b_0). 
\end{align*}
Similarly, after $k$ such layers, we get
\begin{align*}
    a_k &\leq a_0 + (a_0 - b_0) \delta^{-d} \eta \sum_{i=0}^{k-1} (1+2 \delta^{-d} \eta)^i,\\
    b_k &\geq b_0 - (a_0 - b_0) \delta^{-d} \eta \sum_{i=0}^{k-1} (1+2 \delta^{-d} \eta)^i,
\end{align*}
showing that the gap $a_k - b_k$ may grow exponentially in the worst case:
\begin{align*}
    a_k - b_k \leq (1+2 \delta^{-d} \eta)^k (a_0 - b_0).
\end{align*}

In the error-less case ($\hdmx$), for any input sequence $\mH$, the maximum possible difference between maximum and minimum of $\vu^T \mH$ is bounded above by $n\delta^{-d}$, and after one selective shift operation was done on the $\gamma(2)$-th column, the difference is then bounded by $n\delta^{-2d}$.
Therefore, the worst-case possible error introduced by $\rho$ is bounded above by the sum of the worst-case errors calculated assuming that we started off with max-min difference $n\delta^{-2d}$.
Using this observation, the error on each first-row entry of the sequence after the first $p\delta^{-d}$ layers is bounded above by
\begin{equation}
\label{eq:accerror}
    2 n\delta^{-2d} \cdot \delta^{-d} \eta \sum_{i=0}^{\delta^{-d}-1} (1+2 \delta^{-d} \eta)^i,
\end{equation}
where a factor of $2$ is introduced because when the selective shift operation is applied to the $\gamma(2)$-th column, it may introduce an error which is twice the magnitude of the error introduced to the other columns.
We want to make \eqref{eq:accerror} smaller than $\frac{\delta}{8n}$. By Assumption~\ref{assm:rhotohdmx}, we can always choose $t > 0$ that satisfies the assumption for 
\begin{equation*}
    \zeta = \frac \delta 2, \text{ and }
    \eta = \half  \delta^{2d} \log \left (1+\frac{\delta^{2d}\wt \delta}{8 n^2} \right ) > 0,
    \text{ where }
    \wt \delta \defeq \min \left \{\delta, \frac{2^{1-1/p} \epsilon}{n^{1/p}} \right \}.
\end{equation*}
Using such $t$, we can control the total accumulated error by the first $p\delta^{-d}$ selective shift operations below $\frac{\wt \delta}{8n}$:
\begin{align*}
    &~2 n\delta^{-2d} \cdot \delta^{-d} \eta \sum_{i=0}^{\delta^{-d}-1} (1+2 \delta^{-d} \eta)^i
    \leq
    2n \delta^{-3d} \eta \frac{(1+ 2 \delta^{-d} \eta )^{\delta^{-d}} - 1 } {(1+2\delta^{-d}\eta) - 1}\\
    =&~
    n \delta^{-2d} \left ( \left ( 1+\frac{ \log \left (1+\frac{\delta^{2d}\wt \delta}{8 n^2} \right)}{\delta^{-d}} \right )^{\delta^{-d}} - 1 \right )
    \leq 
    n \delta^{-2d}
    \left ( \exp \log \left (1+\frac{\delta^{2d}\wt \delta}{8 n^2} \right) - 1 \right )\\
    =&~ n \delta^{-2d} \frac{\delta^{2d}\wt \delta}{8 n^2}
    = 
    \frac{\wt \delta}{8n}.
\end{align*}
Therefore, after the first $p\delta^{-d}$ selective shift layers, the accumulated error for each entry of the first row is at most $\wt \delta/8n$.

We can also apply similar arguments to the remaining selective shift layers. For example, for the $j$-th set of $p\delta^{-d}$ selective shift layers where the operation is done on $\gamma(j+1)$-th column of the input, the gap between the maximum and the minimum, including the accumulated error from previous layers, is bounded above by $n\delta^{-(j+1)d}$. 
Therefore, for this set of layers, the maximum accumulated error is bounded by
\begin{equation*}
    2 n\delta^{-(j+1)d} \cdot \delta^{-d} \eta \sum_{i=0}^{\delta^{-d}-1} (1+2 \delta^{-d} \eta)^i.
\end{equation*}
So, choosing $t>0$ that satisfies Assumption~\ref{assm:rhotohdmx} for $\eta = \frac \delta 2$ and $\eta = \half \delta^{2d} \log (1+\frac{\delta^{(j+1)d} \wt \delta}{8n^2})$, we can control the accumulated error introduced by the $p\delta^{-d}$ layers below $\frac{\delta}{8n}$:
\begin{align*}
    &~2 n\delta^{-(j+1)d} \cdot \delta^{-d} \eta \sum_{i=0}^{\delta^{-d}-1} (1+2 \delta^{-d} \eta)^i
    \leq 
    2n \delta^{-(j+2)d} \eta \frac{(1+ 2 \delta^{-d} \eta )^{\delta^{-d}} - 1 } {(1+2\delta^{-d}\eta) - 1}\\
    \leq&~
    n \delta^{-(j+1)d} \left ( \left ( 1+\frac{ \log \left (1+\frac{\delta^{(j+1)d}\wt\delta}{8 n^2} \right)}{\delta^{-d}} \right )^{\delta^{-d}} - 1 \right )
    \leq \frac{\wt \delta}{8n}.
\end{align*}
In total, the accumulated error by the first $p(n-1)/\delta^d$ layers, which correspond to the selective shift operation part of the construction, is at most $\frac{(n-1)\wt \delta}{8n} \leq \frac{\wt \delta}{8}$.

\paragraph{Errors in all-max-shift operations.}
For all-max-shift operations, we approximate the hardmax $\hdmx$ all-max-shift operations $\Omega^l(\mZ; n\delta^{-nd})$ with its $\rho$-counterparts, $\wt \Omega^l(\mZ; n\delta^{-nd})$.
We can similarly bound the accumulated error in the all-max-shift operations. Recall from \S~\ref{sec:proof-lem-contextmap} that during the whole series of all-max-shift operations, the maximum entry in the sequence is upper-bounded by $(2sn\delta^{-nd-1})^s n\delta^{-nd}$ and minimum entry is lower-bounded by $(n-1)\Delta$. Therefore, the gap between the max and min elements, taking into consideration the errors from selective shift operations, is bounded from above by $(2sn\delta^{-nd-1})^s n\delta^{-nd}$.
Then, using a similar argument as the select shift operation layers, the maximum error is bounded above by
\begin{equation*}
    (2sn\delta^{-nd-1})^s n\delta^{-nd} \cdot n \delta^{-nd} \eta \sum_{i=0}^{s-1} ( 1 + n \delta^{-nd} \eta )^i,
\end{equation*}
and we want to make it smaller than $\frac{\wt \delta}{8}$. By Assumption~\ref{assm:rhotohdmx}, we can always choose $t > 0$ that satisfies the assumption for 
\begin{equation*}
    \zeta = \frac \delta 2, \text{ and }
    \eta = \frac{\delta^{nd}}{sn} \log \left ( 1+\frac{\delta^{s(nd+1)+nd}\wt \delta}{2^{s+3}s^s n^{s+1}} \right ) > 0.
\end{equation*}
Using such $t$, we can control the total accumulated error by the first $p\delta^{-d}$ selective shift operations below $\frac{\wt \delta}{8}$:
\begin{align*}
    &~(2sn\delta^{-nd-1})^s n\delta^{-nd} \cdot n \delta^{-nd} \eta \sum_{i=0}^{s-1} ( 1 + n \delta^{-nd} \eta )^i\\
    \leq&~
    (2sn\delta^{-nd-1})^s n\delta^{-nd} \cdot n \delta^{-nd} \eta \frac{(1+ n \delta^{-nd}\eta )^{s} - 1} {(1+ n \delta^{-nd}\eta) - 1}\\
    =&~
    (2sn\delta^{-nd-1})^s n\delta^{-nd} \left ( \left ( 1+\frac{ \log \left (1+ \frac{\delta^{s(nd+1)+nd}\wt \delta}{2^{s+3}s^s n^{s+1}}\right)}{s} \right )^{s} - 1 \right )\\
    \leq&~
    (2sn\delta^{-nd-1})^s n\delta^{-nd}
    \frac{\delta^{s(nd+1)+nd}\wt \delta}{2^{s+3}s^s n^{s+1}}
    = \frac{\wt \delta}{8}.
\end{align*}
So far, we have analyzed the total accumulated error of approximating the contextual mapping function $g_{\rm c}$ (constructed with hardmax $\hdmx$) with an approximation $\wt g_{\rm c}$ (constructed with $\rho$). 
We have seen that for any input $\mH \in \sH_\delta$, the approximation error can be controlled so that the error by the selective shift operation part is at most $\wt \delta/8$ and the all-max-shift operation part is at most $\wt \delta/8$.
Therefore, the total error of the $(j,k)$-th entry can be bounded as
\begin{align*}
    \wt g_{\rm c}(\mH)_{j,k} - g_{\rm c} (\mH)_{j,k}
    \in
    \begin{cases}
    [-\frac {\wt \delta} {4}, \frac {\wt \delta} {4}] & j = 1,\\
    \{0\} & j \neq 1,
    \end{cases}
\end{align*}
for any $\mH \in \sH_{\delta}$.

\subsection{Approximating the value mapping $g_{\rm v}$ (Lemma~\ref{lem:valuemap})}

We now consider the approximation of the value mapping $g_{\rm v}$ with standard feed-forward layers. 
In $g_{\rm v}$, we implemented the function with layers of the form 
\begin{equation*}
    \mZ \mapsto \mZ + (\overline f (\mH-\mE)_{k} - g_{\rm c}(\mH)_{k}) \phi' (\vu^T \mZ - \vu^T g_{\rm c}(\mH)_{k} \vone_n^T),
    ~~
    \phi'(t) = 
    \begin{cases}
    0 & t < -\delta/2 \text{ or } t \geq \delta/2,\\
    1 & -\delta/2 \leq t < \delta/2.
    \end{cases}
\end{equation*}
Since the output of contextual mapping $g_{\rm c}(\mH)$ and its approximation $\wt g_{\rm c}(\mH)$ differ in only the first row and by $\wt \delta/4 \leq \delta / 4$, one can approximate each layer in $g_{\rm v}$ by replacing $\phi'$ with an approximation $\wt \phi'$, implementable with four $\relu$'s:
\begin{equation*}
    \wt \phi'(t) = 
    \begin{cases}
    0 & t < -\delta/2 \text{ or } t \geq \delta/2,\\
    \frac{4}{\delta} t + 2 & -\delta/2 \leq t < -\delta/4,\\
    1 & -\delta/4 \leq t < \delta/4,\\
    -\frac{4}{\delta} t + 2 & \delta/4 \leq t < \delta/2.
    \end{cases}
\end{equation*}
Let $\wt g_{\rm v}$ be the approximation of $g_{\rm v}$ constructed this way.
Because the error on $\wt g_{\rm c}$ is bounded by $\wt \delta/4$, the error on the final output $\wt g_{\rm v}$ is also bounded by $\wt \delta/4$. That is, for any $\mH \in \sH_\delta$,
\begin{equation*}
    \wt g_{\rm v} (\wt g_{\rm c}(\mH))_{j,k} - g_{\rm v}(g_{\rm c} (\mH))_{j,k}
    \in
    \begin{cases}
    [-\frac {\wt \delta} {4}, \frac {\wt \delta} {4}] & j = 1,\\
    \{0\} & j \neq 1.
    \end{cases}
\end{equation*}
Hence, using $\wt \delta \defeq \min \left \{\delta, \frac{2^{1-1/p} \epsilon}{n^{1/p}} \right \}$, we have
\begin{equation*}
    \norm{\wt g_{\rm v} ( \wt g_{\rm c} (\mH) ) - g_{\rm v} ( g_{\rm c} (\mH) )}_p^p \leq n \Big ( \frac{\wt\delta}{4} \Big )^p
    \leq \frac{1}{2} \left ( \frac{\epsilon}{2} \right )^p,
\end{equation*}
for all $\mH \in \sH_\delta$.

\subsection{Finishing the proof}
Recall from \S~\ref{sec:proof-lem-step3-quantize} that the approximated quantization function $\wt g_{\rm q}$ maps most of the input $\mX \in \sD$ to $\mH \in \sH_\delta$, and a small fraction of them (of measure at most $1- (1-\alpha)^{nd}$) to something else. Note now that the original function $\overline g = g_{\rm v} \circ g_{\rm c} \circ g_{\rm q}$ and the approximation $g = \wt g_{\rm v} \circ \wt g_{\rm c} \circ \wt g_{\rm q}$ are both bounded, so there is a global constant $B$ such chat $\norm{\overline g(\mX + \mE) - g(\mX + \mE)}_p \leq B$ for all $\mX \in \sD$. 

We can divide the integral over $\sD$ to two disjoint sets. The first one $\sD_1 \defeq \{ \mX \in \sD \mid \wt g_{\rm q}(\mX+\mE) \in \sH_\delta \}$ is the set of input $\mX$ mapped to $\sH_\delta$ by $\wt g_{\rm q}$, and the other is its complement $\sD_2 = \sD \setminus \sD_1$.
\begin{align*}
    \funcdist_p(\overline g, g)^p
    &\defeq
    \int_{\sD} \norm{\overline g(\mX+\mE) - g(\mX+\mE)}_p^p d \mX\\
    &= 
    \int_{\sD_1} \norm{\overline g(\mX+\mE) - g(\mX+\mE)}_p^p d \mX
    +
    \int_{\sD_2} \norm{\overline g(\mX+\mE) - g(\mX+\mE)}_p^p d \mX\\
    &\leq
    \frac{1}{2} \left ( \frac{\epsilon}{2} \right )^p + (1 - (1-\alpha)^{nd}) B^p.
\end{align*}
One can make $\alpha$ close enough to 1 so that the second term is less than $\frac{1}{2} \left ( \frac{\epsilon}{2} \right )^p$. This makes $\funcdist_p(\overline g, g) \leq \epsilon/2$, hence finishing the proof.

\section{Experimental setup}
\label{sec:expsetup}
\subsection{Copying task}
We generated the synthetic dataset for the copying task. The input sequence to the copying task has the format $0\mathbf{s}0\mathbf{s}$, where $\mathbf{s}$ is a 127 length sequence of symbols randomly sampled from the range of $[0, 127]$. The training set contains 100K sequences, while the testing set contains 10K sequences.

{We implement the copying task as a masked-LM \citep{devlin2018bert} style prediction task by masking all the tokens in the second half of the sequence. For the test examples, each masked token is predicted independently. For the results reported in \S~\ref{sec:exp}, we experiment with bidirectional models, where each token can attend to both previous and future tokens.}

The maximum sequence length is $n = 256$, and we use embedding dimension $d=256$.
{The model has 1 to 4 attention layers with $h=4$ attention heads of size $m = 64$, followed by a feed-forward hidden layer of size $r=512$.} We train the model with the AdamW optimizer with weight decay and no dropout. We train the model using 3,000 warmup steps and a total of 500K training steps. The learning rate is $1e^{-4}$. We use the batch size 1,024 on 8 TPUv3 chips.

For all sparsity patterns other than the \textsc{Random} pattern, we choose the segment length $w$ to be 16 for all patterns. This segment length results in the sparsest level for the \textsc{Strided} and \textsc{Fixed} patterns. In Table~\ref{tab:synthetic}, we include the sparsity level as a reference.
For this task, we report the prediction accuracy for all the tokens.

\subsection{Language modeling}
For the language modeling task, we train on the One Billion Word Benchmark \citep{lm1b} which contains almost one billion tokens and a vocabulary of more than 800K tokens.

We use the Transformer model in the Tensor2Tensor framework \citep{vaswani2018tensor2tensor}. We use a 12-block (cf.~\eqref{eq:s-ff}) Transformer, with embedding dimension $d=256$, maximum sequence length $n=256$, number of heads $h = 8$, head size $m = 64$, and feed-forward hidden layer size $r = 1024$. Since language modeling task is auto-regressive (attending to only past tokens) in nature, we evaluate the (sparse) attention score matrices and mask them to be an upper-triangular matrix.
We train the model with the Adafactor with weight decay. We train the model using 10K warmup steps and a total of 240K steps. We use the batch size 4,096 on 8 TPUv2 chips.

For this task, we report the perplexity.

\subsection{Translation}
For the translation task, we train on the WMT18 en-cs datasets (Europarl v7, Common Crawl corpus, News Commentary v13, and CzEng), with a total of 15M pairs of sentences, and test on the newstest2015 en-cs dataset, with 2,656 pairs.

{We use the encoder-decoder architecture and apply the sparse attention on both encoder and decoder.}
We use the Transformer model in the Tensor2Tensor framework \citep{vaswani2018tensor2tensor} and the same setup as the language modeling task, except for having 6 blocks in the Transformer networks, with head size $m = 32$ and having autoregressive patterns only in decoders.

For this task, we report the cased BLEU score.

{
\subsection{GLUE tasks}
For the GLUE tasks, we use the pre-training and fine-tuning framework \citep{devlin2018bert}. Following \citet{devlin2018bert} we first pre-train a $\BB$ model for 450K steps on the BooksCorpus \citep{zhu2015aligning} (800M words) and the English Wikipedia datasets (2,500M words). We later finetune the model on data from each task separately. For each setting, we use the same sparsity pattern and head configuration in both the pre-training and the fine-tuning stages. The sequence length is $n=128$ in both stages.}

{
We report the average accuracy of three runs on the dev set for all tasks. For each setting, we pre-train a model and run fine-tuning three times.
}

\begin{table*}[t!]
    \centering
    \caption{Accuracy on the synthetic copying task when using an auto-regressive model. Percentages in parentheses mark the sparsity levels.}
    \setlength{\tabcolsep}{3pt}
    \scalebox{0.9}{
    \begin{tabular}{ccccccccc}
    \toprule
    \multicolumn{1}{l}{} &
    \multicolumn{3}{c}{{\sc Strided}}   &
    \multicolumn{3}{c}{{\sc Fixed}}  & {\sc Star} & {\sc Random}\\
    \cmidrule(lll){2-4} \cmidrule(lll){5-7}
        {\bf Depth}
        & \begin{tabular}{@{}c@{}}{\sc Union} \\ (87\%)\end{tabular}
        & \begin{tabular}{@{}c@{}}{\sc Multihead} \\ (93\%)\end{tabular} & \begin{tabular}{@{}c@{}}{\sc Sequential} \\ (93\%)\end{tabular} &\begin{tabular}{@{}c@{}}{\sc Union} \\ (87\%)\end{tabular}
        & \begin{tabular}{@{}c@{}}{\sc Multihead} \\ (93\%)\end{tabular} & \begin{tabular}{@{}c@{}}{\sc Sequential} \\ (93\%)\end{tabular} & (87\%) & (90\%)\\
        \toprule 
        1-layer &  0.79\% & 0.78\% & 0.78\% & 7.02\% & 7.04\% & 0.81\% & 0.77\% & 33.13\% \\
        2-layer &  12.40\% & 8.26\% & 1.57\% & 73.43\% & 13.24\% & 92.10\% & 12.32\% & 67.30\% \\
        3-layer &  94.50\% & 65.58\% & 60.88\% & 99.87\% & 70.82\% & 99.84\% & 14.03\% & 89.50\% \\
        4-layer &  100\% & 100\% & 98.40\% & 99.97\% & 99.16\% & 99.97\% & 31.19\% & 95.88\% \\
        \bottomrule
    \end{tabular}
    }
    \label{tab:synthetic_unidirectonal}
    \vspace{-0.5\baselineskip}
\end{table*}

\begin{figure}[t!]
    \centering
    \subfloat[WMT en-de]{
    \label{fig:translate_ende}
      \includegraphics[width=0.49\textwidth]{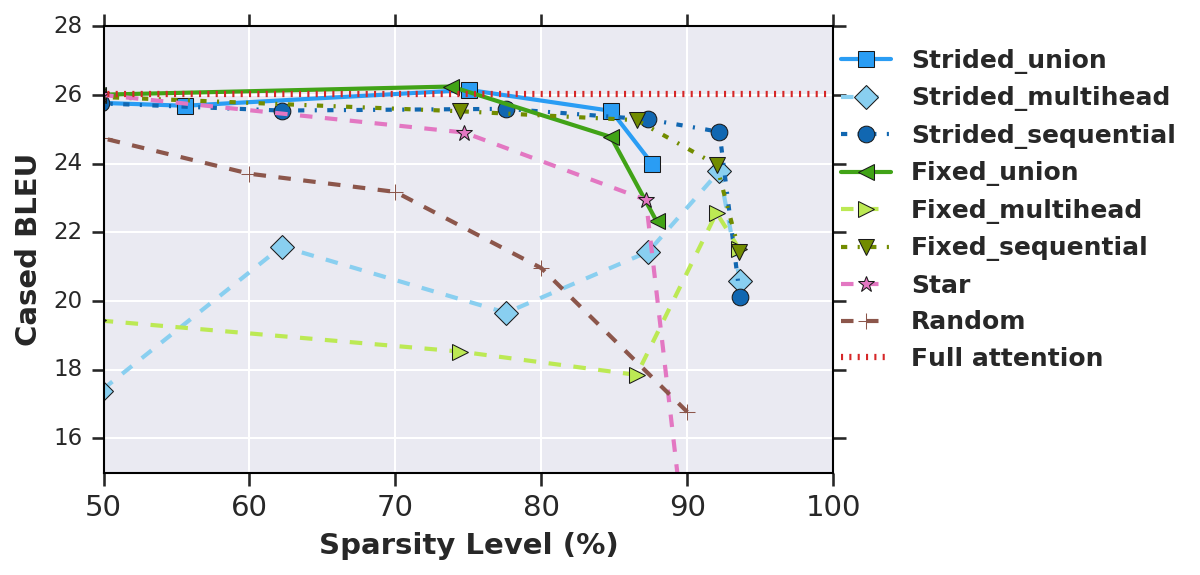}
    }
    \subfloat[WMT de-en]{
    \label{fig:translate_deen}
        \includegraphics[width=0.49\textwidth]{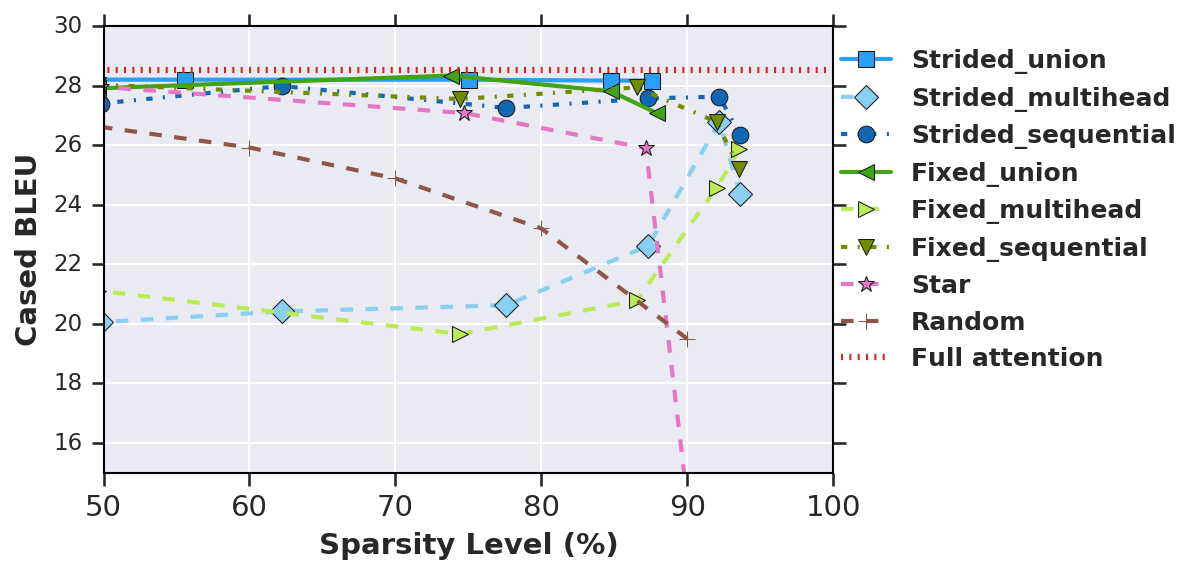}
    }
    \caption{Comparison of sparsity patterns and different head configurations on the WMT de-en and en-de translation tasks.}
    \label{fig:translate_supp}
\end{figure}

\begin{figure}[th!]
    \centering
    \subfloat[CoLA]{
    \label{fig:bert_cola}
      \includegraphics[width=0.49\textwidth]{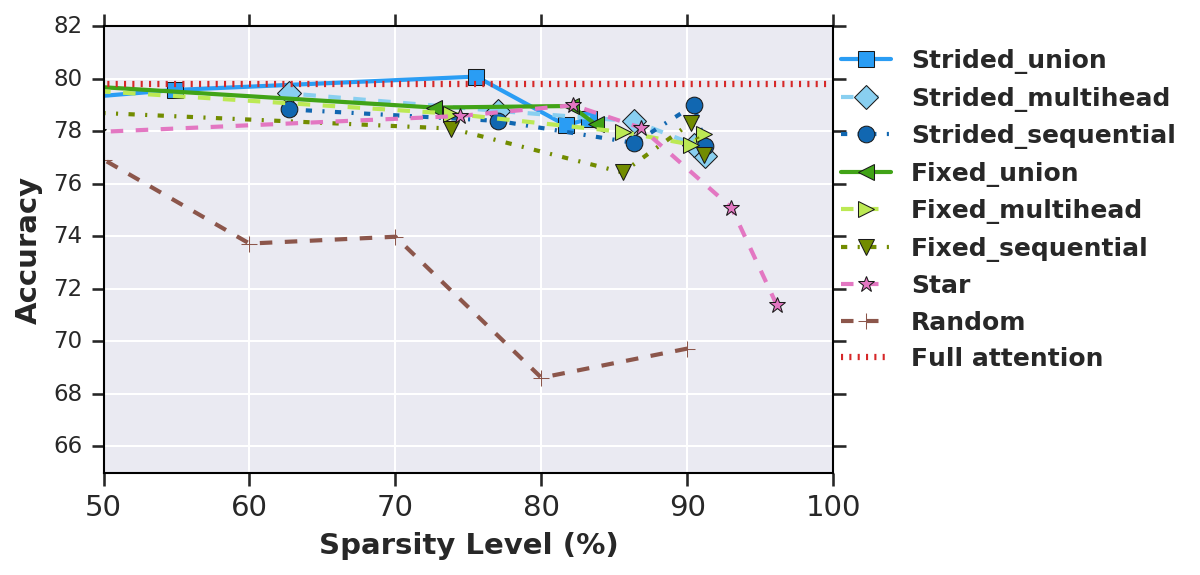}
    }
    \subfloat[MRPC]{
    \label{fig:bert_mrpc}
        \includegraphics[width=0.49\textwidth]{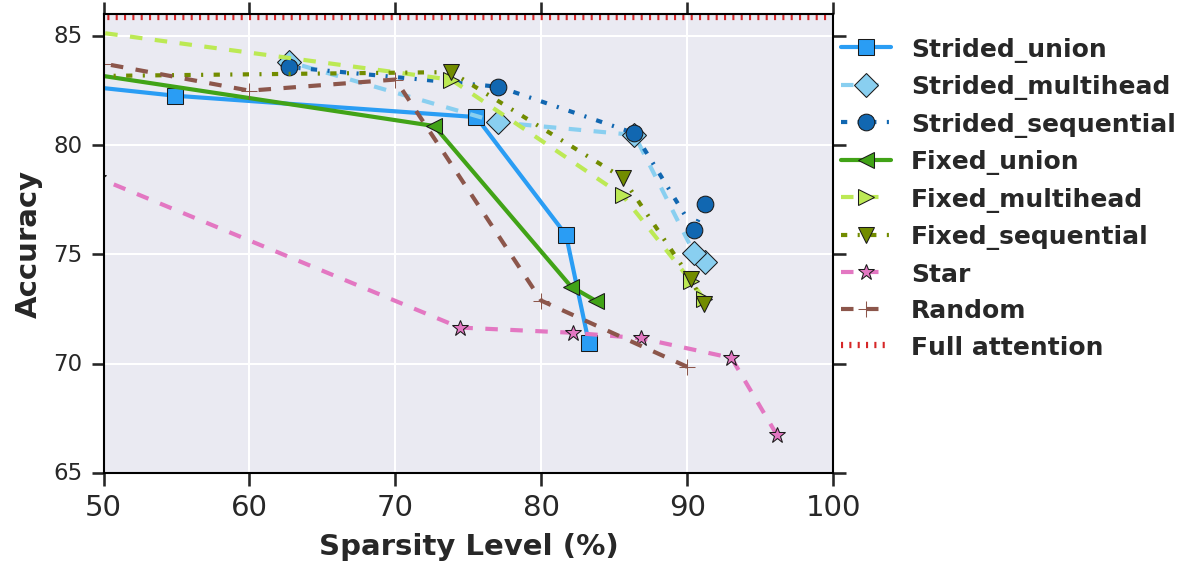}
    }
    \caption{Comparison of sparsity patterns and different head configurations on the CoLA and MRPC tasks for the $\BB$ model.}
    \label{fig:bert_supp}
\end{figure}

\section{Additional experimental results}
\label{sec:additional_exp}
{
We report additional experimental results in this section.
}

{
\subsection{Copying task}
We include the results for the copying task using auto-regressive (unidirectional) models as in LM, where each token can only attend to previous tokens, in Table~\ref{tab:synthetic_unidirectonal}. In this case, the {\sc Star} pattern cannot attend to the last replay token. Indeed, the {\sc Star} pattern shows better performance when the model is bidirectional (cf.~Table~\ref{tab:synthetic}).
}

{
\subsection{Translation}
We present experimental results of the translation tasks on the WMT English-German and German-English datasets in Figure~\ref{fig:translate_supp}. We train on WMT18 (Europarl v7, Common Crawl corpus and News Commentary v13) and test on newstest 2015 datasets. The figures show similar trends to the results on the WMT en-cs dataset in Figure~\ref{fig:translate_encs}.}

{
\subsection{GLUE tasks}
Figure~\ref{fig:bert_supp} presents the results comparing the sparsity patterns and the head configurations on the CoLA and MRPC tasks using the $\BB$ model. CoLA is a single-sentence classification task, asking if a sentence is a grammatical English sentence. MRPC is a sentence-pair classification task, where each example is a pair of sentences and the label indicates whether the sentences are semantically equivalent.
}

\end{document}